\documentclass[letterpaper]{article} 
\usepackage{sty/aaai23}  
\usepackage{times}  
\usepackage{helvet}  
\usepackage{courier}  
\usepackage[hyphens]{url}  
\usepackage{graphicx} 
\usepackage{natbib}  
\usepackage{caption} 
\usepackage{algorithm}
\usepackage[noEnd=true,indLines=true,commentColor=black]{sty/algpseudocodex}
\usepackage{stmaryrd}
\usepackage{newfloat}
\usepackage{listings}
\DeclareCaptionStyle{ruled}{labelfont=normalfont,labelsep=colon,strut=off} 
\floatstyle{ruled}
\newfloat{listing}{tb}{lst}{}
\floatname{listing}{Listing}
\nocopyright

\usepackage{booktabs}
\usepackage{bm}
\usepackage{multirow}
\usepackage{amsmath,amssymb,amsthm}

\newtheorem{theorem}{Theorem}

\usepackage{hhline}
\usepackage[font=footnotesize,subrefformat=parens]{subcaption}
\usepackage{arydshln}
\usepackage{tikz}
\usetikzlibrary{calc}
\usetikzlibrary{positioning}
\usetikzlibrary{shapes.geometric}
\usetikzlibrary{patterns}
\usetikzlibrary{decorations.pathreplacing}
\usetikzlibrary{fit}
\usetikzlibrary{arrows.meta}
\usetikzlibrary{bending}
\tikzset{
  vertex/.style={circle,draw,black,align=center,inner sep=0cm, minimum size=0.15cm,fill=black,anchor=center},
  line/.style={black},
  start-vertex/.style={vertex, minimum size=0.4cm,fill=white},
  goal-vertex/.style={vertex, minimum size=0.4cm,fill=white,rectangle},
}

\usepackage[binary-units]{siunitx}
\usepackage{cleveref}
\usepackage{sty/mystyle}

\title{LaCAM: Search-Based Algorithm for Quick Multi-Agent Pathfinding}
\author{
  Keisuke Okumura
}
\affiliations{
  Tokyo Institute of Technology\\
  Tokyo, Japan\\
  okumura.k@coord.c.titech.ac.jp
}

\begin{document}

\maketitle
\begin{abstract}
  We propose a novel complete algorithm for multi-agent pathfinding (MAPF) called \emph{lazy constraints addition search for MAPF (LaCAM)}.
  MAPF is a problem of finding collision-free paths for multiple agents on graphs and is the foundation of multi-robot coordination.
  LaCAM uses a two-level search to find solutions quickly, even with hundreds of agents or more.
  At the low-level, it searches constraints about agents' locations.
  At the high-level, it searches a sequence of all agents' locations, following the constraints specified by the low-level.
  Our exhaustive experiments reveal that LaCAM is comparable to or outperforms state-of-the-art sub-optimal MAPF algorithms in a variety of scenarios, regarding success rate, planning time, and solution quality of sum-of-costs.
\end{abstract}

\section{Introduction}

The \emph{multi-agent pathfinding (MAPF)} problem~\cite{stern2019def} aims to assign collision-free paths on graphs to each agent, which is the foundation of multi-robot coordination.
In MAPF applications such as fleet operations in automated warehouses~\cite{wurman2008coordinating}, it is necessary to solve MAPF with thousands of agents in a real-time manner.

In general, the design of MAPF algorithms needs to consider a trade-off between quality and speed.
From the quality side, solving MAPF optimally is NP-hard in various criteria~\cite{yu2013structure}.
Although many effective optimal algorithms have been developed, it is still challenging to handle a few hundred of agents in real-time (i.e., short timeframes in this context) even with state-of-the-art methods~\cite{li2021pairwise,lam2022branch}.
From the speed side, sub-optimal algorithms can quickly solve large MAPF instances while compromising solution quality.
However, in real-time applications (i.e., scenarios with deadlines for planning time in this context) or in instances with massive agents such that optimal algorithms never handle, there is no choice but to use sub-optimal algorithms.
In addition, recently proposed frameworks can effectively refine the quality of known MAPF solutions~\cite{okumura2021iterative,li2021anytime}.
Consequently, the development of quick sub-optimal algorithms has practical value.

To this end, this paper proposes a novel algorithm called \emph{lazy constraints addition search for MAPF (LaCAM)}.
From the theoretical side, LaCAM is complete.
It returns a solution for solvable instances, otherwise reports the non-existence.
From the empirical side, we demonstrate that LaCAM can solve a variety of MAPF instances in a very short time, including instances with large maps (e.g., $1491\times 656$ four-connected grid), instances with massive agents (e.g., $10,000$), or dense situations.
For instance, LaCAM solved \emph{all} instances with 400 agents on a $32\times 32$ grid with 20\% obstacles from the MAPF benchmark~\cite{stern2019def}, with a median runtime of \SI{1}{\second}.
In contrast, baseline sub-optimal MAPF algorithms~\cite{silver2005cooperative,standley2010finding,okumura2022priority,li2021eecbs,li2022mapf} mostly failed to solve the instances with the timeout of \SI{30}{\second}.

LaCAM has an easily extensible structure, which comprises a two-level search.
At the high-level, it searches a sequence of configurations, where a configuration is a tuple of locations for all agents.
At the low-level, it searches constraints that specify which agents go where in the next configuration.
Successors at the high-level (i.e., configurations) are generated in a lazy manner while following constraints from the low-level, leading to a dramatic reduction of the search effort.
Similar to the popular CBS algorithm~\cite{sharon2015conflict}, due to its simplicity, we consider that LaCAM can apply to many domains not limited to MAPF such as multi-robot motion planning~\cite{solis2021representation}.

Throughout the paper, we focus on sub-optimal LaCAM.
The discussion of optimal LaCAM appears at the end, together with a qualitative discussion of why LaCAM is quick.
In what follows, we present problem formulation, the algorithm, evaluation, and related work in order.
The supplementary material is available on \url{https://kei18.github.io/lacam}.

{
  \begin{figure*}
    \includegraphics[width=1\linewidth]{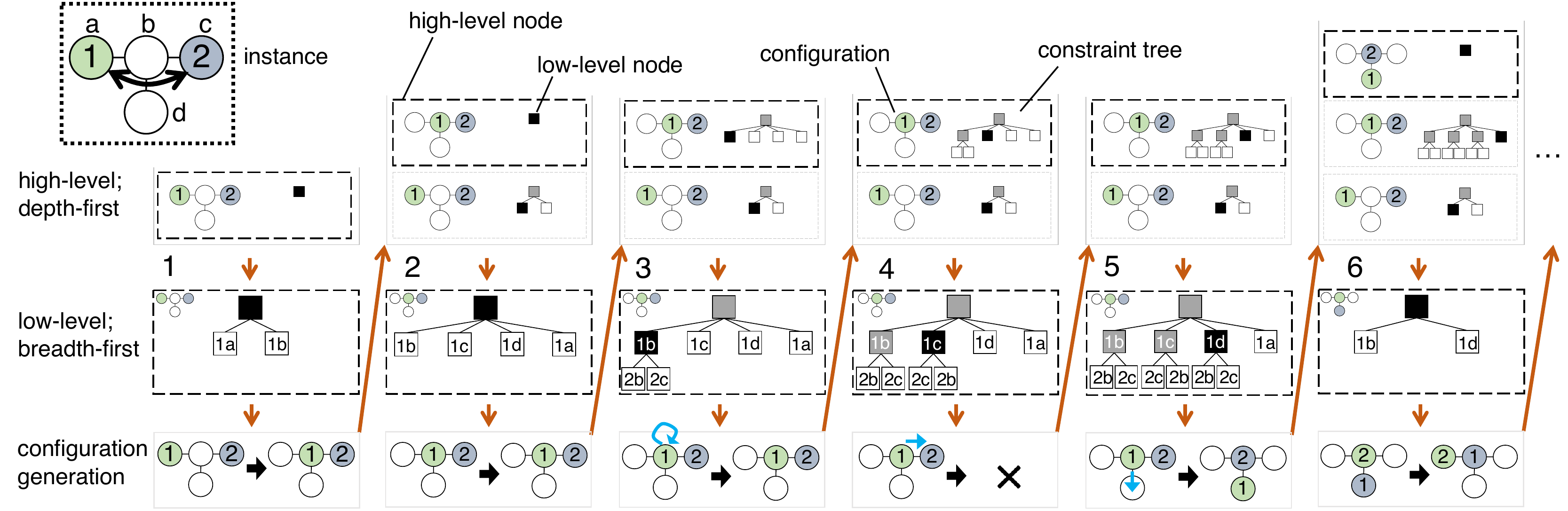}
    \caption{
      \textbf{Running example of LaCAM.}
      Orange arrows represent the search progress order.
      Selected and searched low-level nodes are filled with black and gray, respectively.
      Constraints are shown by blue-colored arrows.
    }
    \label{fig:algo}
  \end{figure*}
}

\section{Preliminaries}
\label{sec:preliminaries}

\paragraph{Problem Definition}
An \emph{MAPF instance} is defined by a graph $G = (V, E)$, a set of agents $A = \{1, \ldots, n\}$, a tuple of distinct starts $\mathcal{S} = (s_1, \ldots, s_n)$ and goals $\mathcal{G} = (g_1, \ldots, g_n)$, where $s_i, g_i \in V$.

Given an MAPF instance, let $\loc{i}{t} \in V$ denote the location of an agent $i$ at discrete time~$t \in \mathbb{N}_{\geq 0}$.
At each timestep~$t$, $i$ can move to an adjacent vertex, or can stay at its current vertex, i.e., $\loc{i}{t+1} \in \neigh{\loc{i}{t}} \cup \{ \loc{i}{t} \}$, where \neigh{v} is the set of vertices adjacent to $v \in V$.
Agents must avoid two types of conflicts:
1)~\emph{vertex conflict}: $\loc{i}{t} = \loc{j}{t}$, and,
2)~\emph{swap conflict}: $\loc{i}{t} = \loc{j}{t+1} \land \loc{i}{t+1} = \loc{j}{t}$.
A set of paths is \emph{conflict-free} when no conflicts exist.

A \emph{solution} to MAPF is a set of conflict-free paths $\{ \path{1}, \ldots, \path{n} \}$ such that all agents reach their goals at a certain timestep~$T \in \mathbb{N}_{\geq 0}$.
More precisely, the problem assigns a path $\path{i} = (\loc{i}{0}, \loc{i}{1}, \dots, \loc{i}{T})$ to each agent such that $\loc{i}{0} = s_i$, $\loc{i}{T}=g_i$, and is conflict-free.

This paper considers a \emph{sum-of-costs (SOC)} metric to rate solution quality: $\sum_{i \in A} T_i$, where $T_i \leq T$ is the earliest timestep such that $\loc{i}{T_i} = \loc{i}{T_i+1} = \ldots = \loc{i}{T} = g_i$.

\paragraph{Notations}
We use $S[k]$ to denote the $k$-th element of the sequence $S$, where the index starts at one.
$\Delta$ is the maximum degree of the graph.
A \emph{configuration} refers to a tuple of locations for all agents, e.g., the start and goal configurations are $\mathcal{S}$ and $\mathcal{G}$, respectively.
Two configurations $X$ and $Y$ are \emph{connected} when $\{ (X[1], Y[1]), \ldots, (X[n], Y[n]) \}$ constitutes a set of conflict-free paths.
For convenience, we use $\bot$ as an ``undefined'' or ``not found'' sign.

\section{Algorithm}
\label{sec:algorithm}
This section first presents the concept of LaCAM, followed by the pseudocode and implementation details.

\subsection{Concept}
LaCAM is a two-level search.
At the high-level, it explores a sequence of configurations; each search node corresponds to one configuration.
For each high-level node, it also performs a low-level search that creates \emph{constraints}.
A constraint specifies which agent is where in the next configuration.
The low-level search proceeds lazily, creating a minimal successor each time the corresponding high-level node is invoked.
\Cref{fig:algo} presents an illustration of LaCAM.
The example MAPF instance is depicted on the left upper side.
The following part explains the figure step by step.

\paragraph{High-Level Search}
As in general search schemes, LaCAM progresses by updating an \OPEN list that stores the high-level nodes.
\OPEN is implemented by data structures of stack, queue, or priority queue.
We use the stack throughout the paper.
Thus, LaCAM is explained as a depth-first search style.
The first row of \cref{fig:algo} shows \OPEN.
For each search iteration, LaCAM selects one node from \OPEN.
Different from general search schemes, LaCAM does not immediately discard the selected node, as explained after three paragraphs.

\paragraph{Low-Level Search}
Each high-level node comprises a configuration and a \emph{constraint tree}.
The constraint tree gradually grows each time invoking the high-level node; this is the low-level search of LaCAM.
Throughout the paper, we use a breadth-first search for the low-level.
The middle row of \cref{fig:algo} visualizes this step.
Each node of the constraint tree has a constraint, except for the root node.
For instance, in the first column, the root node has two successors: `1a' and `1b.'
This means that agent-1 must go to vertex-a or vertex-b in the next configuration.
Successors of the low-level search are created by two steps:
1)~Select an agent $i$. Let $v$ be the vertex of $i$ in the configuration.
2)~Create successors that specifies $i$ is on $u \in \neigh{v}$ or $v$.
The agent is selected so that each path from each low-level node to the root does not contain duplicated agents.
Therefore, those paths specify constraints for several agents.
In addition, no successors are created when the depth of the node is beyond $|A|$ because constraints have been assigned for all agents.

\paragraph{Configuration Generation}
Once both the high- and low-level nodes are specified, a new configuration is generated.
The new one must satisfy the constraints of the low-level node, which are specified by a path to the root node.
Excluding that, any connected configuration from the original configuration can be generated.
\emph{We complement how to generate new configurations following constraints later in \cref{sec:implementation}, but for now, regard this as a black-box function.}
The generation step is visualized in the third row of \cref{fig:algo}.
According to the new configuration, a new high-level node is created.
For instance, at the end of the first column of \cref{fig:algo}, a new configuration $(b, c)$ is generated and inserted to \OPEN.

\paragraph{Discarding High-Level Nodes}
When finished searching all low-level nodes, the corresponding high-level node has been generating all configurations connected to its configuration.
Therefore, this high-level node is discarded.

\paragraph{Example}
LaCAM continues the above search operations until finding the goal configuration $\mathcal{G}$.
Once $\mathcal{G}$ is found, it is trivial to obtain a solution by backtracking high-level nodes.
We next describe \cref{fig:algo} in detail step by step of columns.
\begin{enumerate}
\item Initially, \OPEN contains only a high-level node for the start configuration $\mathcal{S}$.
  At the low-level, two successors are created.
  In this step, any agent can be selected; we choose agent-1.
  Next, LaCAM generates a configuration connected to the original one $(a, c)$.
  Since the target low-level node is the root, there is no constraint.
  Assume that a new configuration $(b, c)$ is generated.
  Then, the corresponding new high-level node is inserted into \OPEN.
\item The high-level node generated in the previous iteration is selected.
  At the low-level, we again choose agent-1.
  This time, LaCAM generates four successors: `1a', `1b,' `1c,' and `1d.'
  We add them to the low-level tree in order of `1b,' `1c,' `1d,' and `1a' to make the example interesting.
  Assume that the same configuration $(b, c)$ is generated.
  Then, a high-level node is not created since the generated configuration has already appeared in the search.
\item The same high-level node is selected as the previous iteration.
  The low-level search generates two nodes for agent-2: `2b' and `2c.'
  This time, the configuration generation must follow the constraint of `1b.'
  Consequently, the same configuration $(b, c)$ is generated and no new high-level node is created.
\item According to the selected low-level node, it is impossible to generate a connected configuration due to conflicts;
  this iteration skips the creation of a high-level node.
\item The constraint makes agent-1 move to vertex-d.
  Then, a new configuration $(d, b)$ is generated.
  The corresponding high-level node is created and inserted into \OPEN.
\item The high-level node for $(d, b)$ is selected, and then, a new configuration $(b, a)$ is generated.
  The search can find the goal configuration in the next iteration.
\end{enumerate}

\subsection{Pseudocode}
\Cref{algo:planner} shows an example implementation of LaCAM.
In the pseudocode, \N and \C correspond to high- and low-level nodes, respectively.
The low-level search uses queue (\tree) because it is breadth-first.
Several details are below.

\paragraph{Configuration Generation}
This is performed by a blackbox function \funcname{get\_new\_config} [\cref{algo:planner:new-config}].
The function returns a configuration connected to a configuration of a high-level node, following constraints specified by a low-level node.
It returns $\bot$ when failing to generate such configurations (e.g., the fourth column of \cref{fig:algo}).
Note that, at the bottom of the low-level tree, all agents have constraints.
Therefore, exactly one configuration is specified without freedom.

\paragraph{High-Level Node Management}
To manage already known configurations, \cref{algo:planner} uses an \EXPLORED table that takes a configuration as a key and stores a high-level node.

\paragraph{Low-Level Agent Selection}
To generate low-level search trees, a high-level node includes \order, an enumeration of all agents sorted by specific criteria, specified by two functions \funcname{get\_init\_order} [\cref{algo:planner:init-node}] and \funcname{get\_order} [\cref{algo:planner:successor}].
The agent is selected following \order and depth of the low-level search tree (starting at one; obtained by a function \funcname{depth}) [\cref{algo:planner:selection}].
Doing so ensures that each path of the constraint tree has no duplicate agents.

{
  \newcommand{\config}{\m{\mathit{config}}}
  \newcommand{\depth}{\m{\mathit{depth}}}
  \newcommand{\parent}{\m{\mathit{parent}}}
  \newcommand{\who}{\m{\mathit{who}}}
  \newcommand{\where}{\m{\mathit{where}}}
  \begin{algorithm}[t!]
    \caption{LaCAM}
    \label{algo:planner}
    \begin{algorithmic}[1]
    \item[\textbf{input}:~MAPF instance ($\mathcal{S}$: starts, $\mathcal{G}$: goals)]
    \item[\textbf{output}:~solution or \texttt{NO\_SOLUTION}]
    \item[\textbf{preface}:~
      $\C\sub{init} \defeq \langle~\parent: \bot, \who: \bot, \where: \bot~\rangle$]
      \State initialize \OPEN, \EXPLORED
      \Comment{\OPEN: stack}
      \State $\N\sub{init} \leftarrow \begin{aligned}[t]\langle~
        &\config: \mathcal{S}, \tree: \llbracket~\C\sub{init}~\rrbracket,
        \hspace{0.8cm}\triangleright~\tree: \text{queue}\\
        &\order: \funcname{get\_init\_order}(), \parent: \bot~\rangle
      \end{aligned}$
      \label{algo:planner:init-node}
      \State $\OPEN.\funcname{push}(\N\sub{init})$;~~$\EXPLORED[\mathcal{S}] = \N\sub{init}$
      \smallskip
      \While{$\OPEN \neq \emptyset$}
      \State $\N \leftarrow \OPEN.\funcname{top}()$
      \IfSingle{$\N.\config = \mathcal{G}$}{\Return $\funcname{backtrack}(\N)$}
      \IfSingle{$\N.\tree = \emptyset$}{$\OPEN.\funcname{pop}()$;~\Continue}
      \State $\C \leftarrow \N.\tree.\funcname{pop}()$
      \smallskip
      \If{$\funcname{depth}(\C) \leq |A|$}
      \State $i \leftarrow \N.\order[\funcname{depth}(\C)]$;\;\;$v \leftarrow \N.\config[i]$
      \label{algo:planner:selection}
      \For{$u \in \neigh{v} \cup \{ v \}$}
      \State $\C\sub{new} \leftarrow \langle~\parent: \C, \who: i, \where: u~\rangle$
      \State $\N.\tree.\funcname{push}(\C\sub{new})$
      \EndFor
      \EndIf
      \State $\Q\sub{new} \leftarrow \funcname{get\_new\_config}(\N, \C)$
      \label{algo:planner:new-config}
      \IfSingle{$\Q\sub{new} = \bot$}{\Continue}
      \IfSingle{$\EXPLORED[\Q\sub{new}] \neq \bot$}{\Continue}
      \label{algo:planner:closed}
      \State $\N\sub{new} \leftarrow \begin{aligned}[t]\langle~
        &\config: \Q\sub{new}, \tree: \llbracket~\C\sub{init}~\rrbracket,\\
        &\order: \funcname{get\_order}(\Q\sub{new}, \N), \parent: \N~\rangle\end{aligned}$
        \label{algo:planner:successor}
      \State $\OPEN.\funcname{push}(\N\sub{new})$;~~$\EXPLORED[\Q\sub{new}] = \N\sub{new}$
      \EndWhile
      \State \Return \texttt{NO\_SOLUTION}
    \end{algorithmic}
  \end{algorithm}
}

\begin{theorem}[completeness]
  \Cref{algo:planner} returns a solution for solvable MAPF instances; otherwise, it reports \texttt{NO\_SOLUTION}.
\end{theorem}
\begin{proof}
  A search space is finite.
  For the high-level, the number of configurations is $O\left(|V|^{|A|}\right)$.
  For the low-level, the number of configurations connected to one configuration is $O\left(\Delta^{|A|}\right)$.
  When the low-level search is finished, the corresponding high-level node has been generating all configurations connected to its configuration.
  Consequently, all \emph{reachable} configurations from the start configuration, defined by transitivity over connections of two configurations, are examined, deriving the theorem.
\end{proof}

{
  \setlength{\tabcolsep}{1.7pt}
  \renewcommand{\arraystretch}{0.8}
  \newcommand{\drawmap}[1]{
    \multicolumn{2}{c}{
      \begin{minipage}{0.13\linewidth}
        \centering
        \includegraphics[width=1.0\linewidth]{fig/raw/map/#1.pdf}
      \end{minipage}
    }
  }
  \begin{table*}[t!]
    \centering
    \begin{tabular}{lrrrrrrrrrrrrc}
      \toprule
      & \multicolumn{2}{c}{\mapname{tree}}
      & \multicolumn{2}{c}{\mapname{corners}}
      & \multicolumn{2}{c}{\mapname{tunnel}}
      & \multicolumn{2}{c}{\mapname{string}}
      & \multicolumn{2}{c}{\mapname{loop-chain}}
      & \multicolumn{2}{c}{\mapname{connector}}
      \\
      \cmidrule(lr){2-3}
      \cmidrule(lr){4-5}
      \cmidrule(lr){6-7}
      \cmidrule(lr){8-9}
      \cmidrule(lr){10-11}
      \cmidrule(lr){12-13}
      & \drawmap{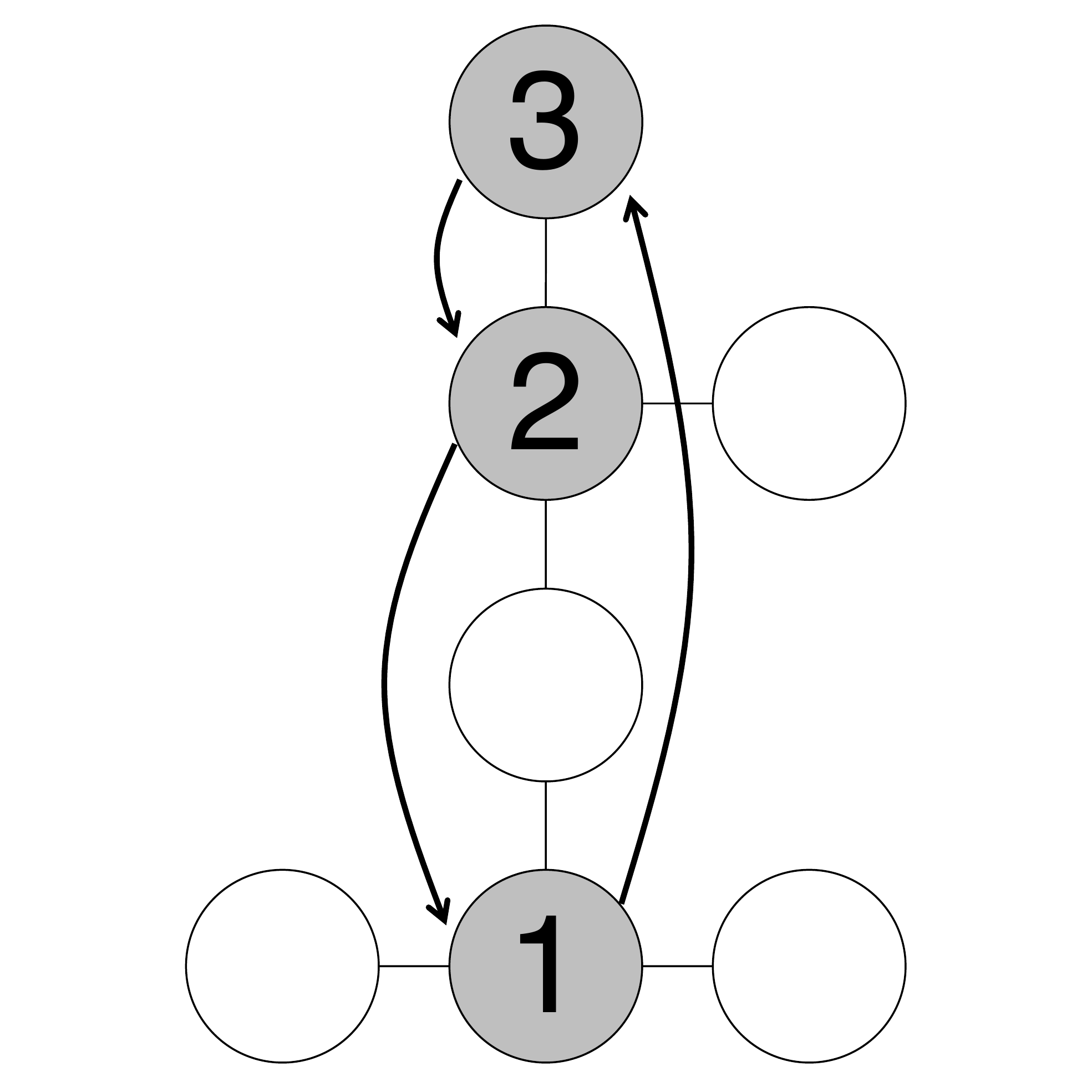}
      & \drawmap{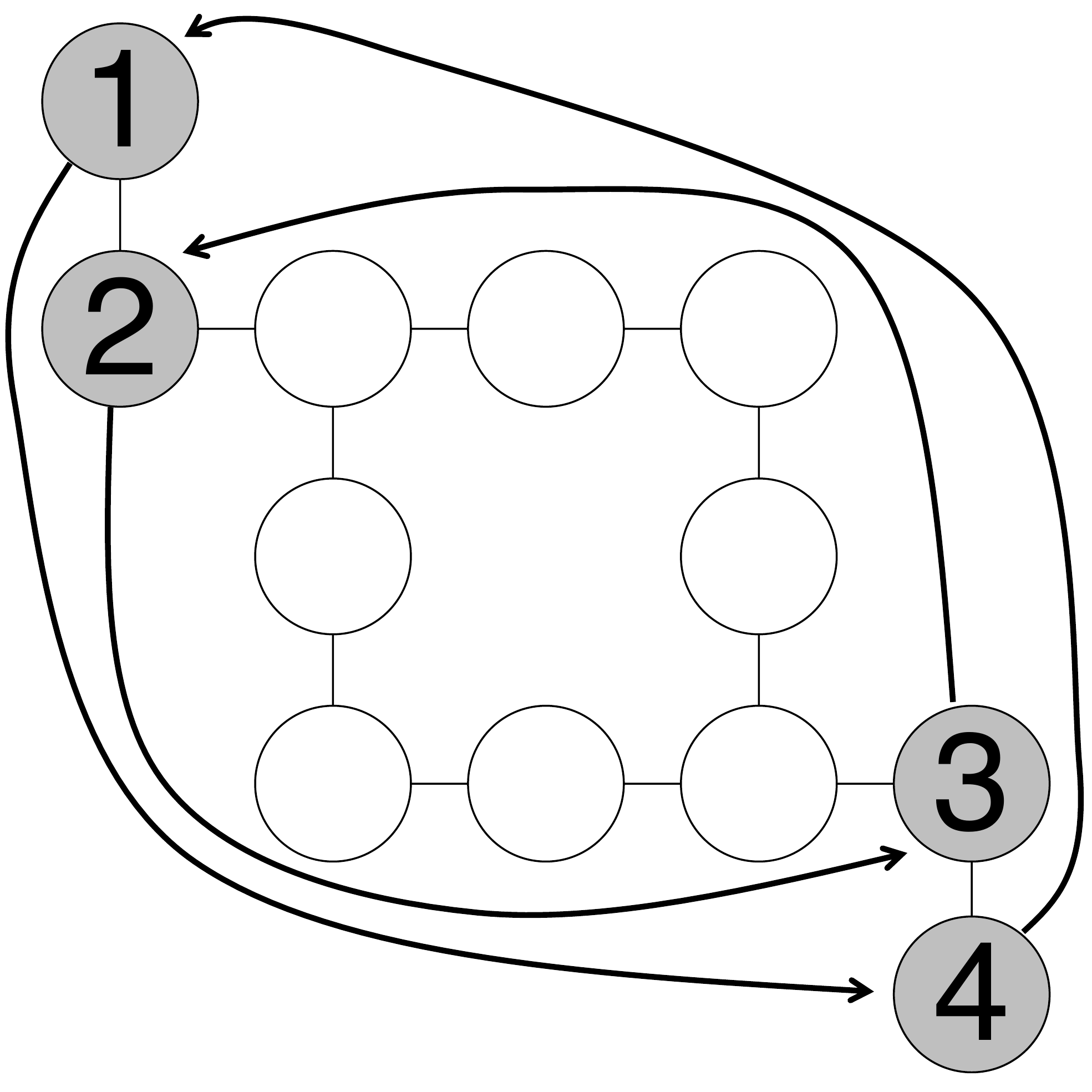}
      & \drawmap{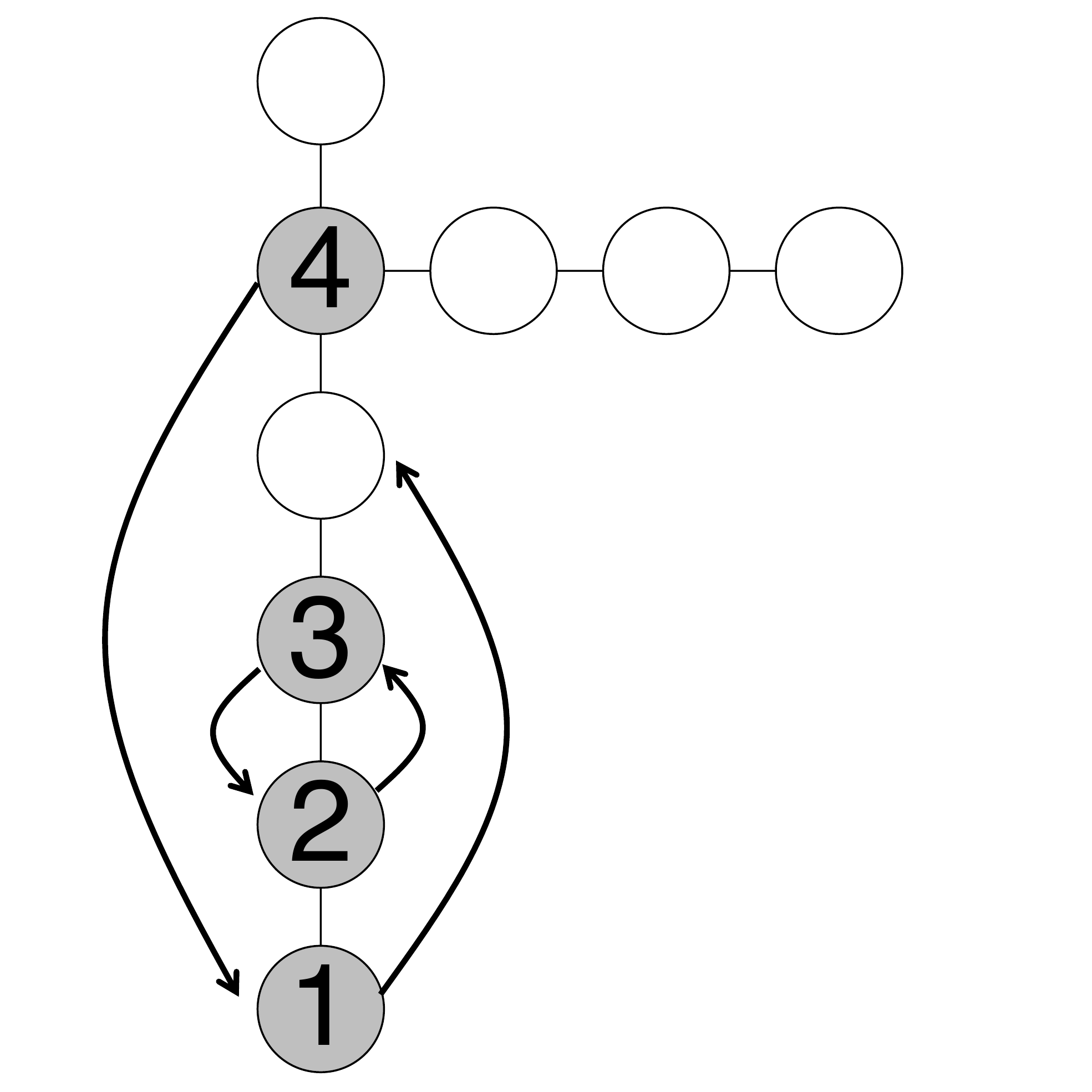}
      & \drawmap{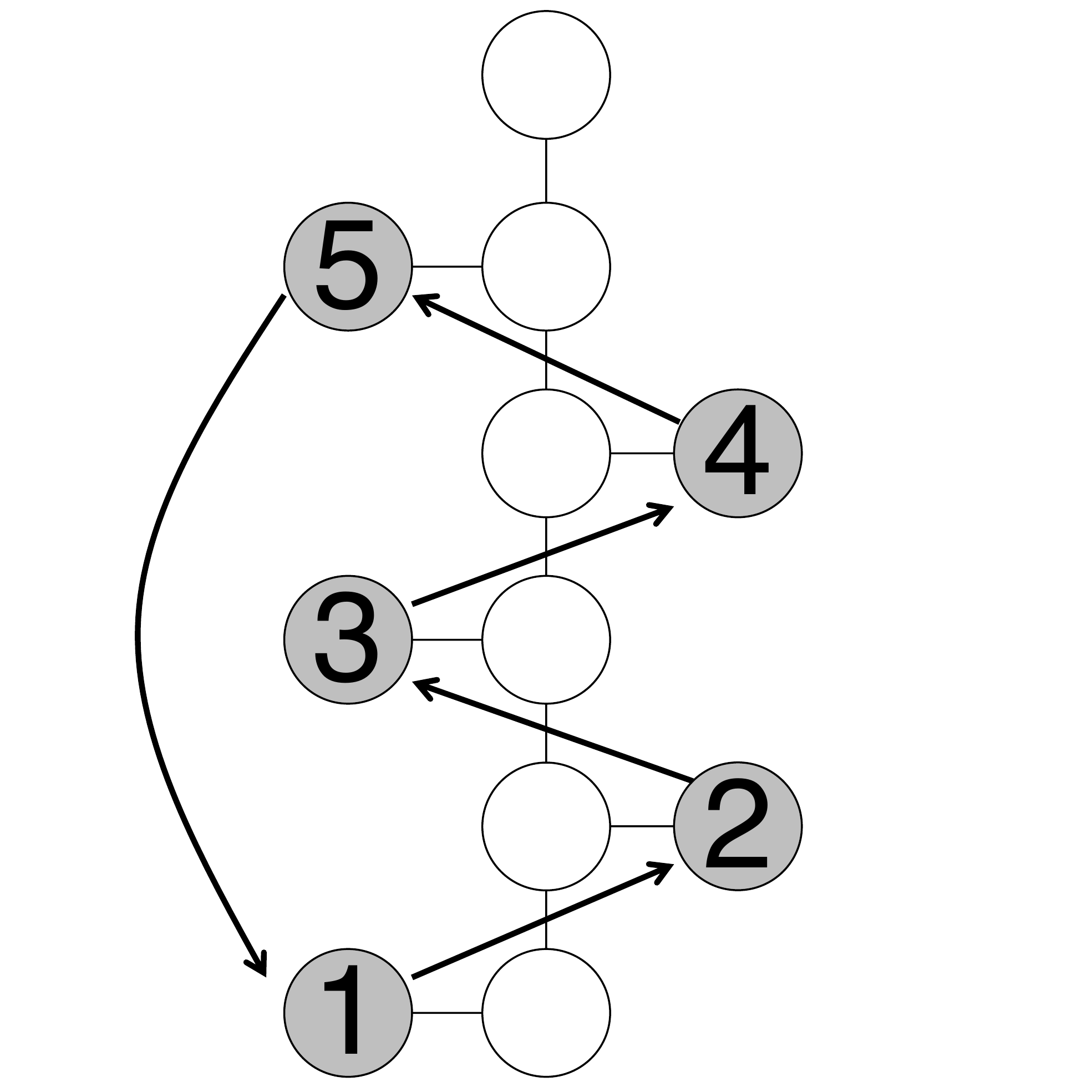}
      & \drawmap{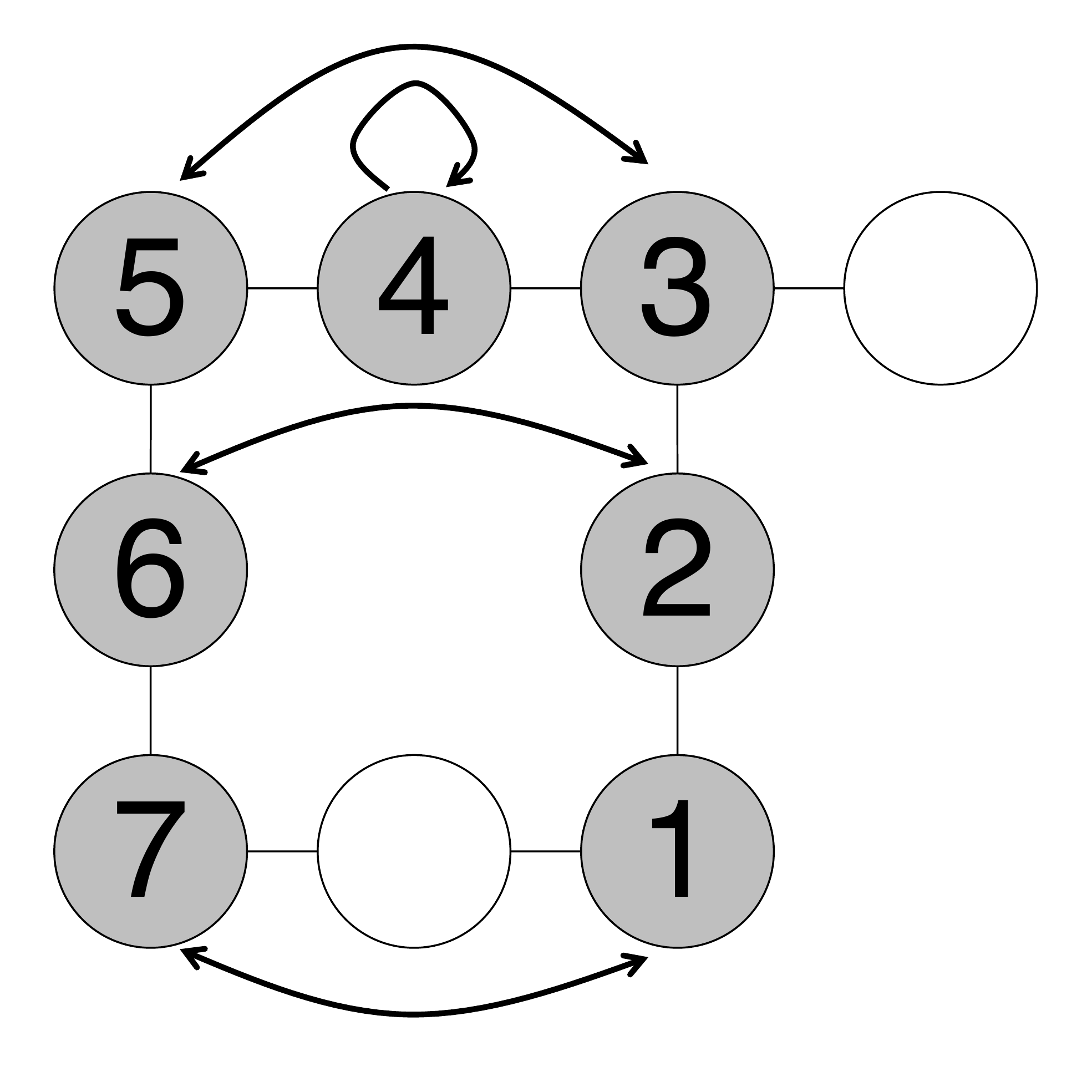}
      & \drawmap{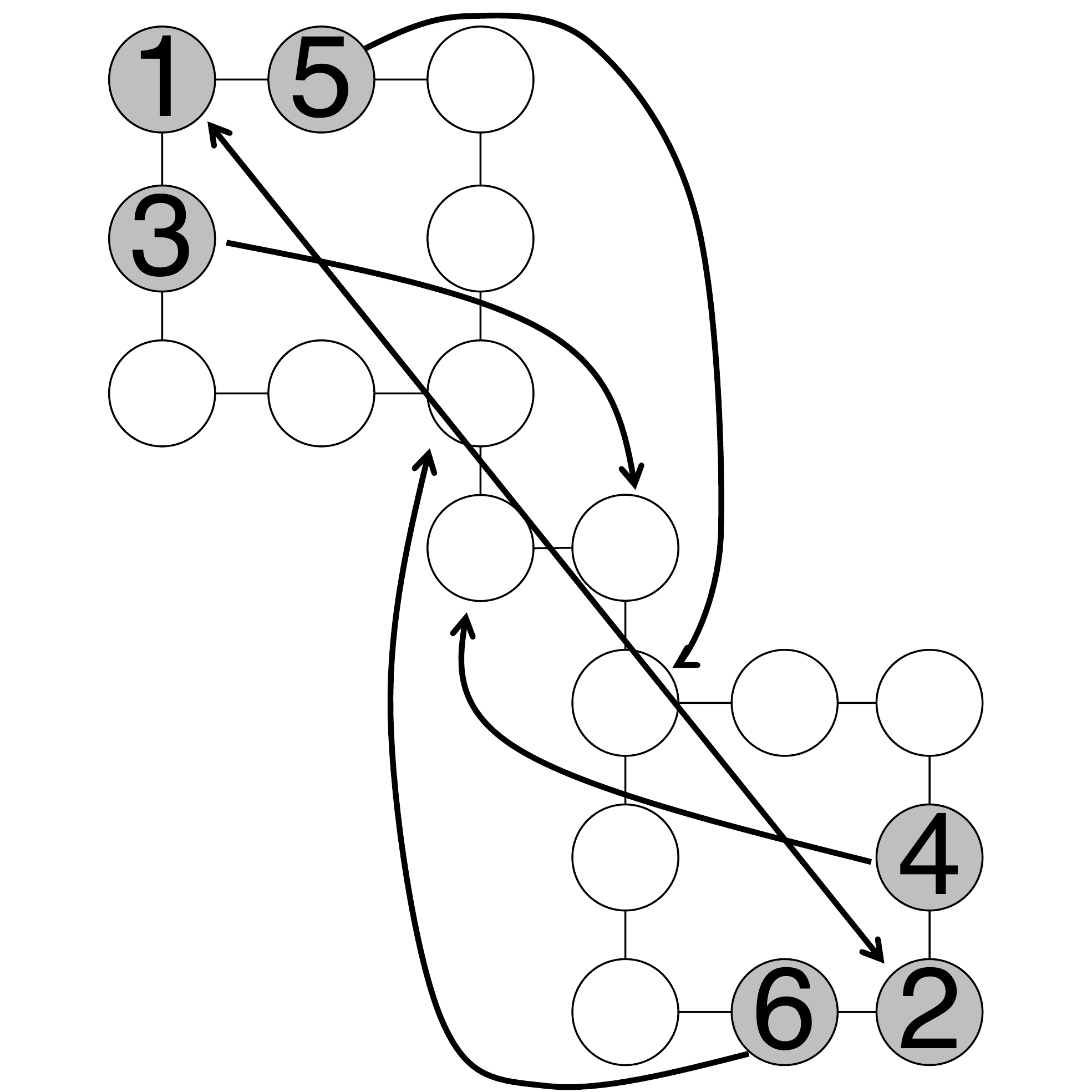}
      \\
      \cmidrule(lr){2-3}
      \cmidrule(lr){4-5}
      \cmidrule(lr){6-7}
      \cmidrule(lr){8-9}
      \cmidrule(lr){10-11}
      \cmidrule(lr){12-13}
      & Time(ms) & SOC
      & Time(ms) & SOC
      & Time(ms) & SOC
      & Time(ms) & SOC
      & Time(ms) & SOC
      & Time(ms) & SOC
      & Solved
      \\
      \cmidrule{1-1}
      \cmidrule(lr){2-3}
      \cmidrule(lr){4-5}
      \cmidrule(lr){6-7}
      \cmidrule(lr){8-9}
      \cmidrule(lr){10-11}
      \cmidrule(lr){12-13}
      \cmidrule(lr){14-14}
      LaCAM{\scriptsize(med)}
      & 0 & 19
      & 17 & 47
      & 173 & 190
      & 0 & 66
      & 34 & 1752
      & 0 & 168
      & \multirow{2}{*}{\textbf{6/6}}
      \\
      LaCAM{\scriptsize(worst)}
      & 0 & 41
      & 31 & 70
      & 208 & 254
      & 0 & 68
      & 124 & 2593
      & 3 & 281
      \\\cmidrule{1-14}
      PP
      & N/A & N/A
      & 0 & 32
      & N/A & N/A
      & N/A & N/A
      & N/A & N/A
      & N/A & N/A
      & 1/6
      \\
      OD
      & 0 & 31
      & 0 & 47
      & 30 & 191
      & 0 & 22
      & 5882 & 2269
      & 27 & 129
      & \textbf{6/6}
      \\
      PIBT
      & N/A & N/A
      & N/A & N/A
      & N/A & N/A
      & N/A & N/A
      & N/A & N/A
      & N/A & N/A
      & 0/6
      \\
      \pibtplus
      & 0 & 55
      & 0 & 54
      & 0 & 227
      & 0 & 52
      & N/A & N/A
      & 0 & 130
      & 5/6
      \\
      EECBS
      & 1 & 16
      & 1 & 32
      & N/A & N/A
      & 0 & 20
      & N/A & N/A
      & N/A & N/A
      & 3/6
      \\
      LNS2
      & N/A & N/A
      & 0 & 32
      & N/A & N/A
      & 0 & 20
      & N/A & N/A
      & 160 & 80
      & 3/6
      \\
      \cmidrule{1-14}
      \astar{\scriptsize(SOC-optimal)}
      & 0 & 16
      & 16 & 32
      & 24 & 53
      & 1 & 20
      & 17752 & 121
      & 391138 & 80
      & \textbf{6/6}
      \\
      \bottomrule
    \end{tabular}
    \caption{
      \textbf{Results of the small complicated instances.}
    }
    \label{table:small-complicated}
  \end{table*}
}

\subsection{Implementation Details}
\label{sec:implementation}

\paragraph{Configuration Generation}
The heart of LaCAM is how to generate configurations following constraints [\cref{algo:planner:new-config}].
Ideally, this sub-procedure should be sufficiently quick and generate a promising configuration to reach the goal configuration.
This can be realized by adapting existing MAPF algorithms that can compute a partial solution, i.e., a set of paths until a certain timestep.
Regarding the target configuration as a start configuration and producing a partial solution by these algorithms, we can extract a configuration at timestep one in the partial solution.
For instance, a naive approach is to adapt prioritized planning with a limited planning horizon, such as windowed \hca~\cite{silver2005cooperative} with single-step window size.
A rolling horizon approach for lifelong MAPF~\cite{li2021lifelong} is also available to generate configurations.
An aggressive approach is to adapt PIBT~\cite{okumura2022priority}, a scalable MAPF algorithm that repeats planning for one-timestep; we used PIBT in the experiments.
\emph{Those algorithms, originally developed to solve MAPF, are available to create promising successors in the high-level search of LaCAM.}

\paragraph{Order of Agents}
In our implementation, the agent order of the initial high-level node was in descending order of the distance between the start and goal [\cref{algo:planner:init-node}], influenced by commonly used heuristics of prioritized planning~\cite{van2005prioritized}.
In other high-level nodes [\cref{algo:planner:successor}], we prioritized agents who are not on their goal, aiming to create constraints for those agents earlier in the low-level search.
A tiebreak used the last arrival timestep of agents in the high-level search so that agents who have been apart from goals for a long time were prioritized, akin to PIBT~\cite{okumura2022priority}.

\paragraph{Order of Low-Level Nodes}
As seen in \cref{fig:algo}, the order of inserting low-level nodes affects search progress.
We provisionally made the order random.
This part requires further investigation in the future.

\paragraph{Reinsert High-Level Node}
\Cref{algo:planner} takes a naive depth-first search style.
Instead, when finding an already known configuration, we observed that reinserting the corresponding high-level node to \OPEN [\cref{algo:planner:closed}] can improve solution quality.
The reason is that repeatedly appearing configurations in the search  can be seen as a bottleneck;
it makes sense to advance the low-level search of the high-level node, which is performed by the reinsert operation.
Note that LaCAM does not lose completeness with this modification.

\section{Evaluation}
\label{sec:evaluation}
This section evaluates LaCAM.
Specifically, we show five empirical results:
1)~evaluation with small complicated MAPF instances,
2)~evaluation with the MAPF benchmark,
3)~showing the current limitation of LaCAM using an adversarial instance,
4)~scalability test with up to 10,000 agents, and
5)~investigating other implementation designs.

{
  \setlength{\tabcolsep}{0pt}
  \newcommand{\entry}[3]{
    \begin{minipage}{0.155\linewidth}
      \centering
      \begin{tabular}{ll}
        \begin{minipage}{0.65\linewidth}
          \baselineskip=7pt
          {\tiny\mapname{#1}}\\
          {\tiny #2 (#3)}
        \end{minipage} &
        \begin{minipage}{0.28\linewidth}
          \includegraphics[width=1\linewidth]{fig/raw/map/#1.pdf}
        \end{minipage}
      \end{tabular}
      \\
      \IfFileExists{fig/raw/success-rate_#1.pdf}
      {\includegraphics[width=1\linewidth]{fig/raw/success-rate_#1.pdf}}
      {\includegraphics[width=1\linewidth]{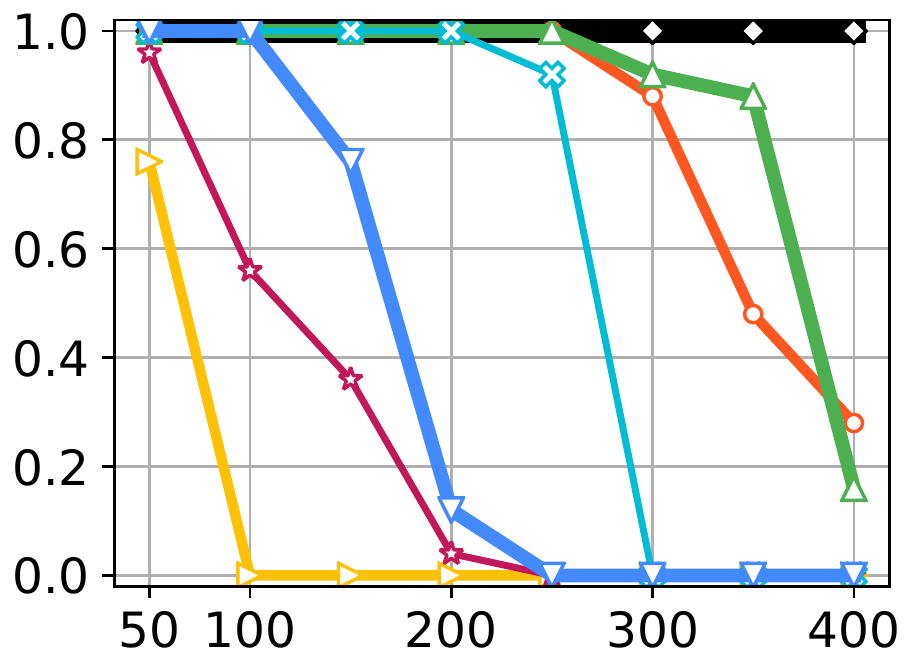}}
      \IfFileExists{fig/raw/runtime_#1.pdf}
      {\includegraphics[width=1\linewidth]{fig/raw/runtime_#1.pdf}}
      {\includegraphics[width=1\linewidth]{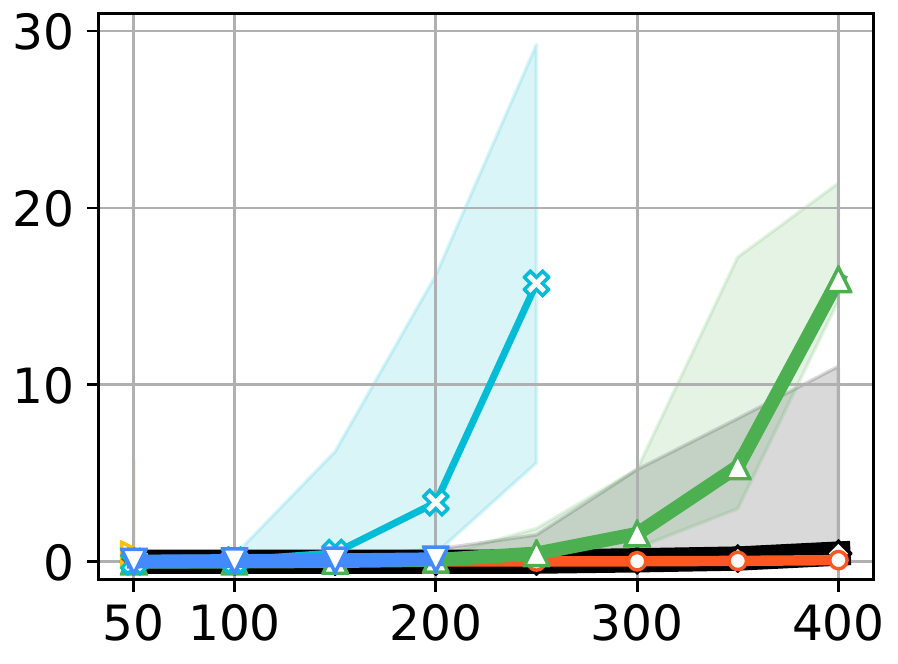}}
      \IfFileExists{fig/raw/sum-of-costs_#1.pdf}
      {\includegraphics[width=1\linewidth]{fig/raw/sum-of-costs_#1.pdf}}
      {\includegraphics[width=1\linewidth]{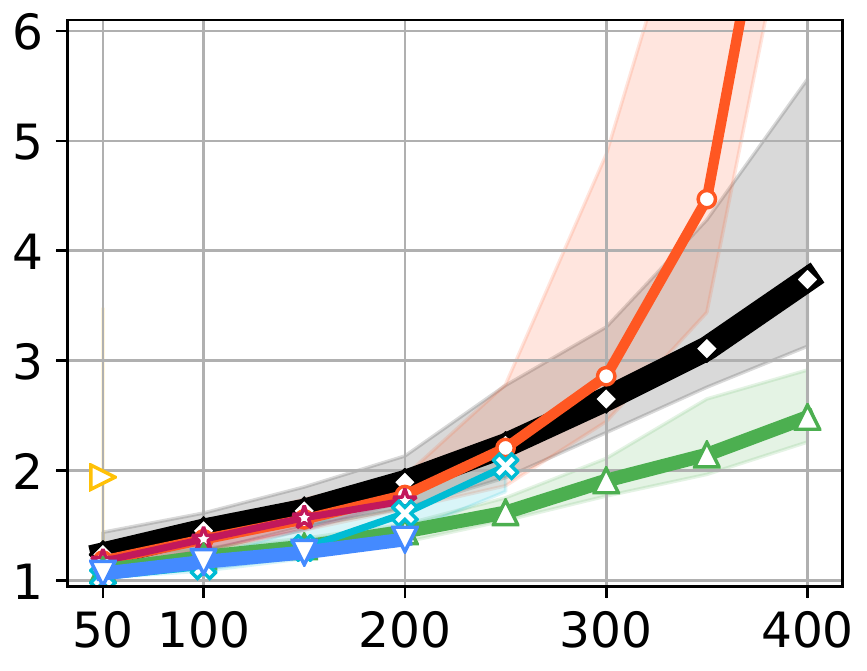}}
    \end{minipage}
  }
  \newcommand{\labels}{
    \begin{minipage}{0.03\linewidth}
      \begin{tikzpicture}
        \node[rotate=90](l1) at (0, 0.1) {cost};
        \node[rotate=90](l2) at (0, 2.1) {runtime (sec)};
        \node[rotate=90](l3) at (0, 4.2) {success rate};
        \node[] at (0, -1) {};
        \node[] at (0, 6) {};
      \end{tikzpicture}
    \end{minipage}
  }
    \begin{figure*}[th!]
    \centering
    \begin{tabular}{ccccccc}
      \labels &
      \entry{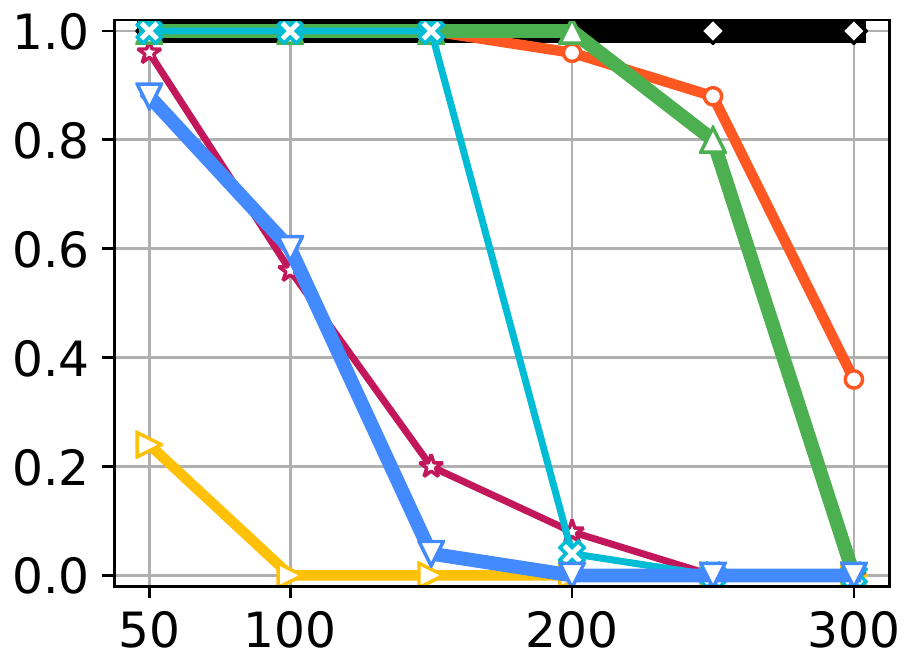}{32x32}{682} &
      \entry{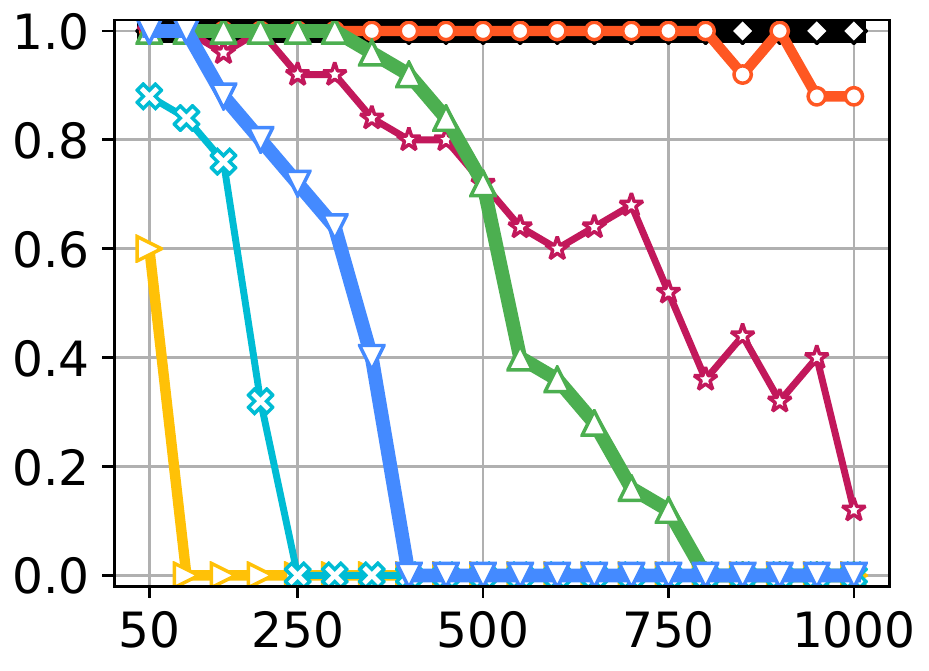}{64x64}{3,646} &
      \entry{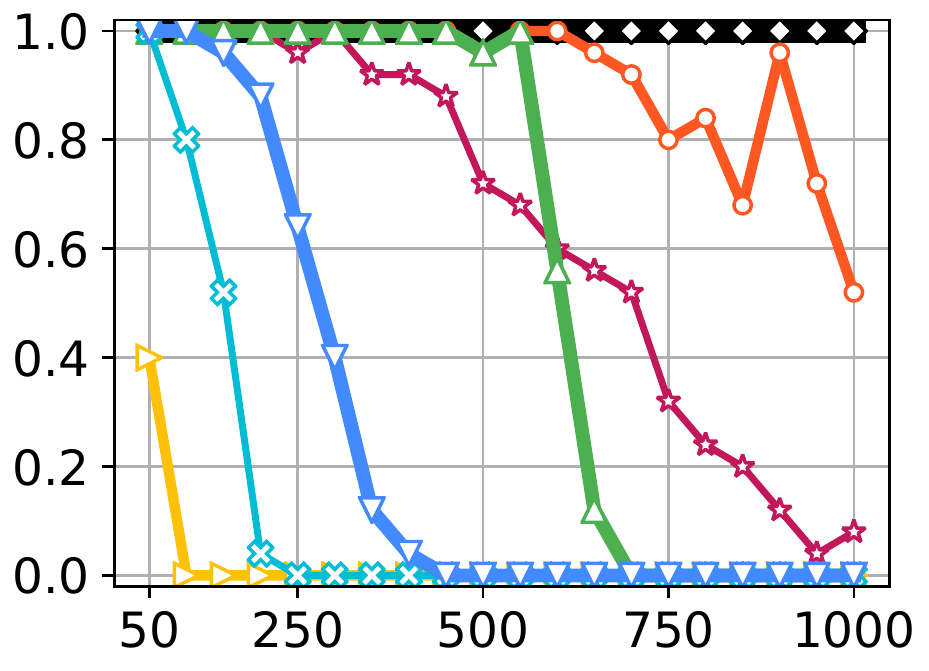}{64x64}{3,232} &
      \entry{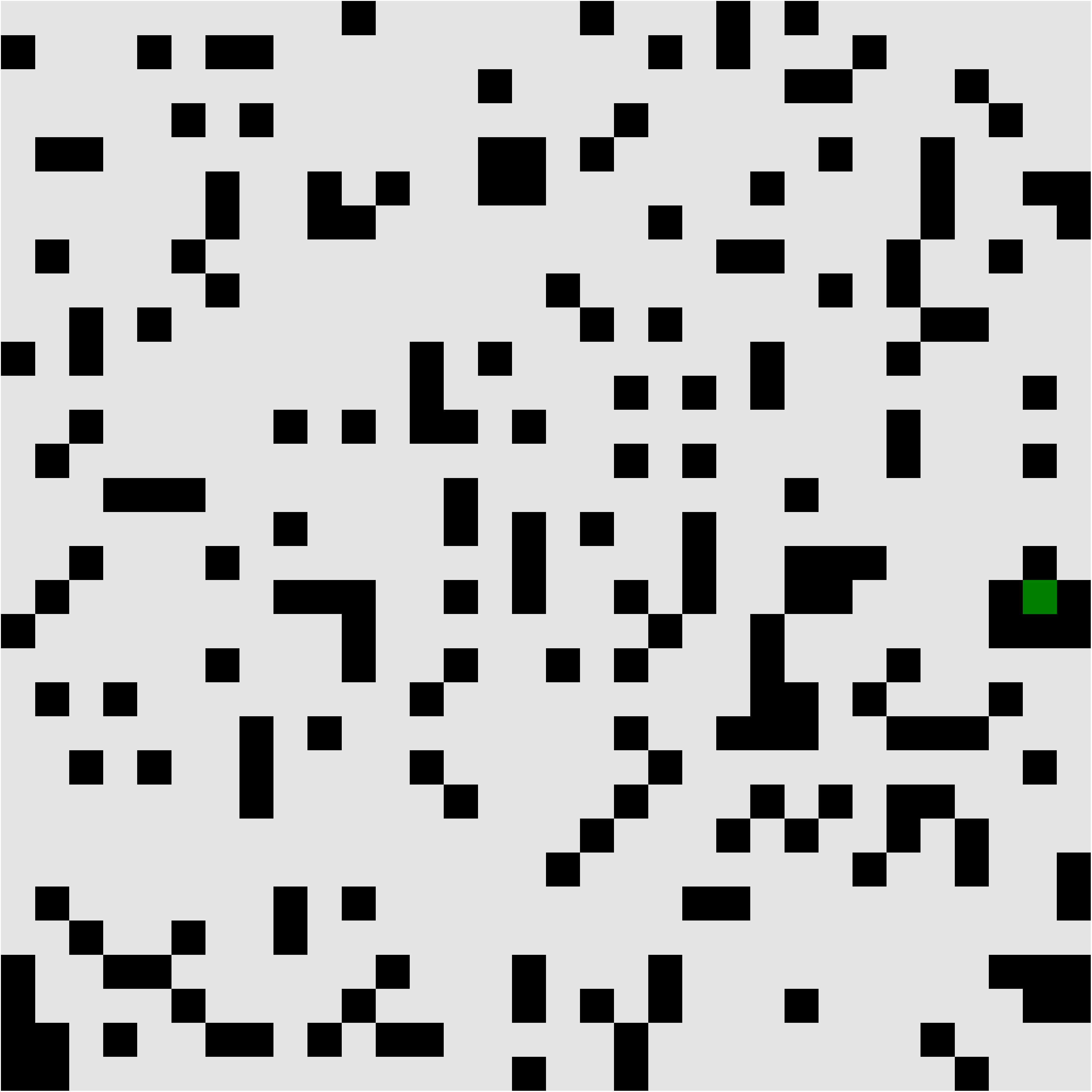}{32x32}{819} &
      \entry{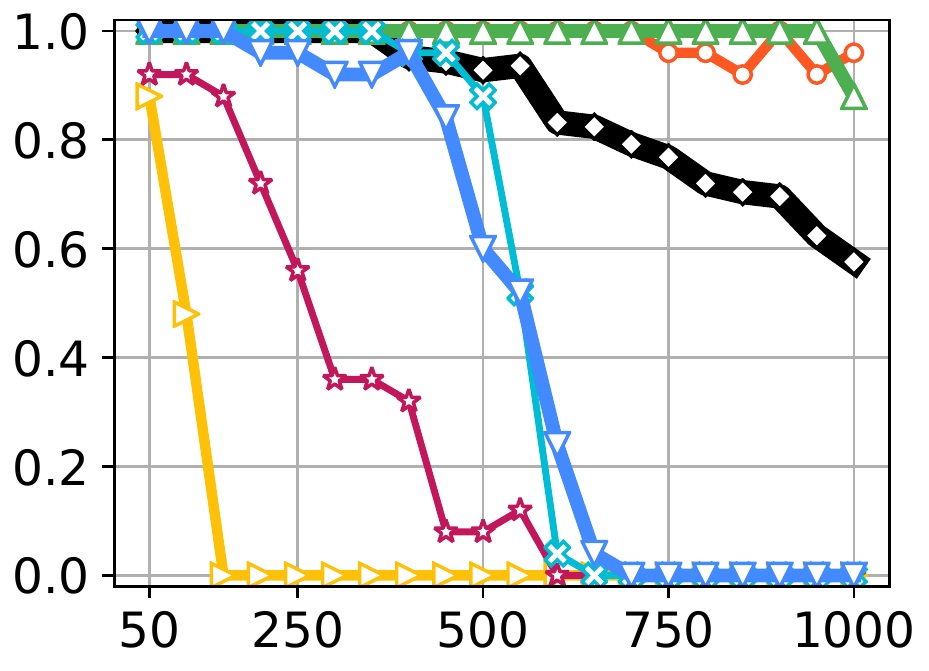}{64x64}{3,270} &
      \entry{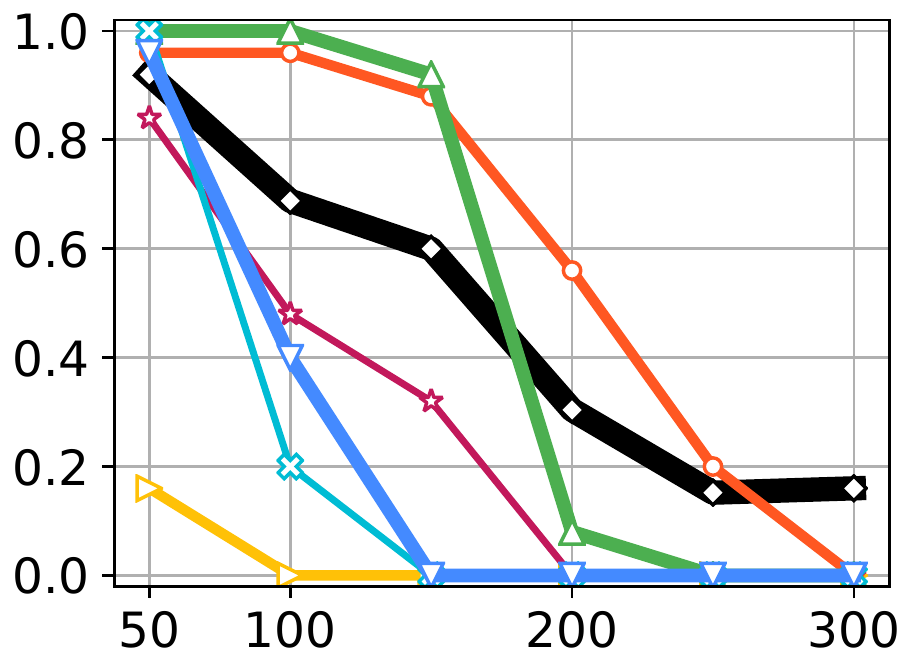}{32x32}{666}
      \smallskip\\
      \labels &
      \entry{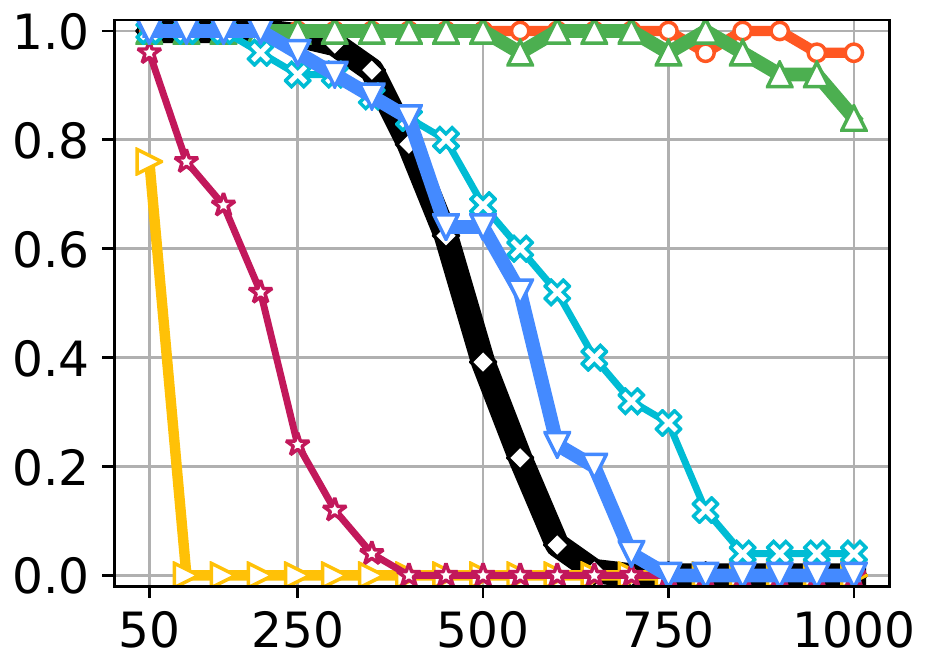}{321x123}{22,599} &
      \entry{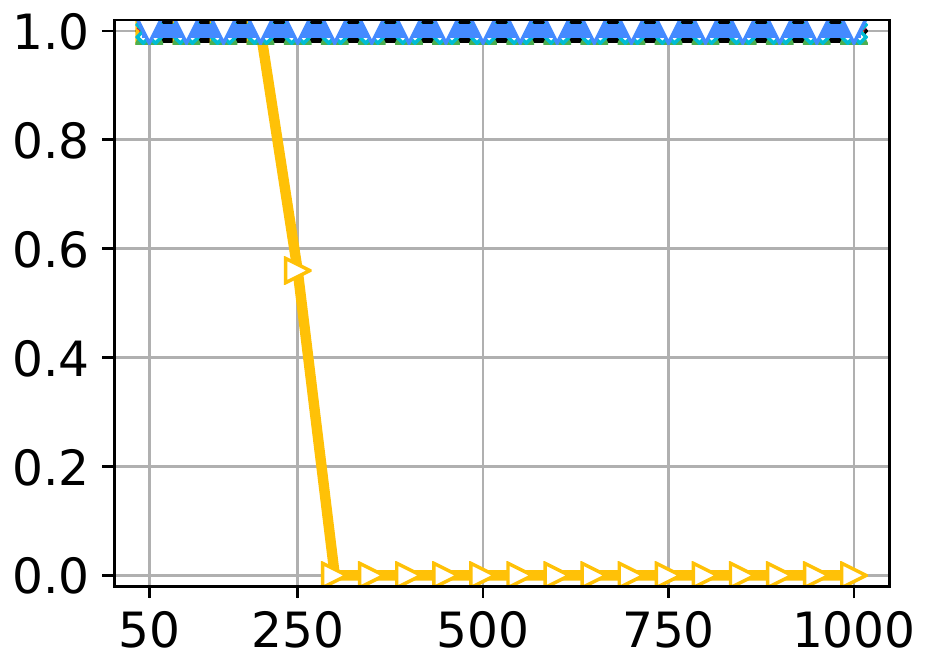}{340x164}{38,756} &
      \entry{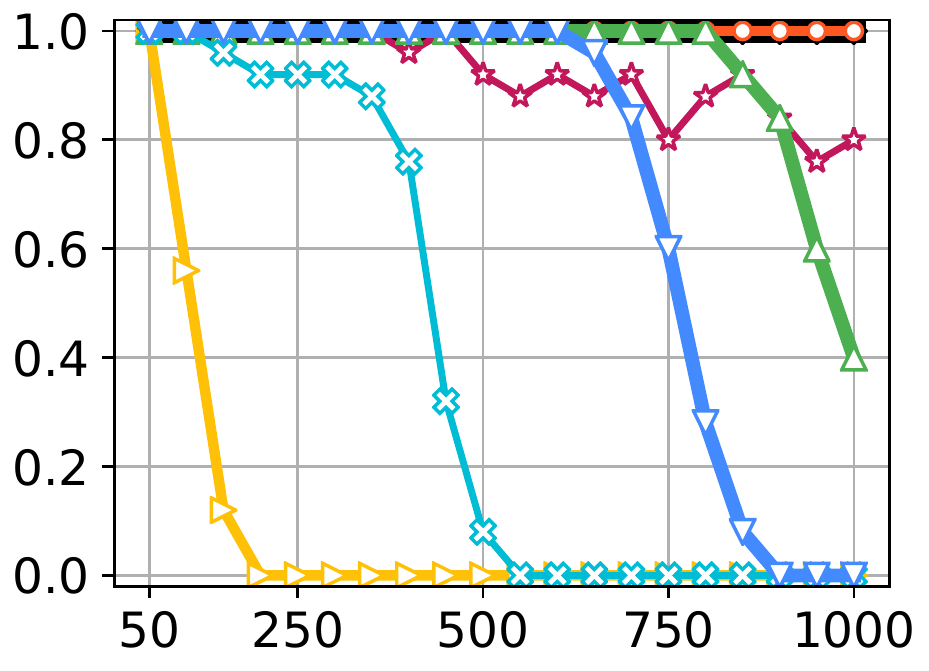}{194x194}{13,214} &
      \entry{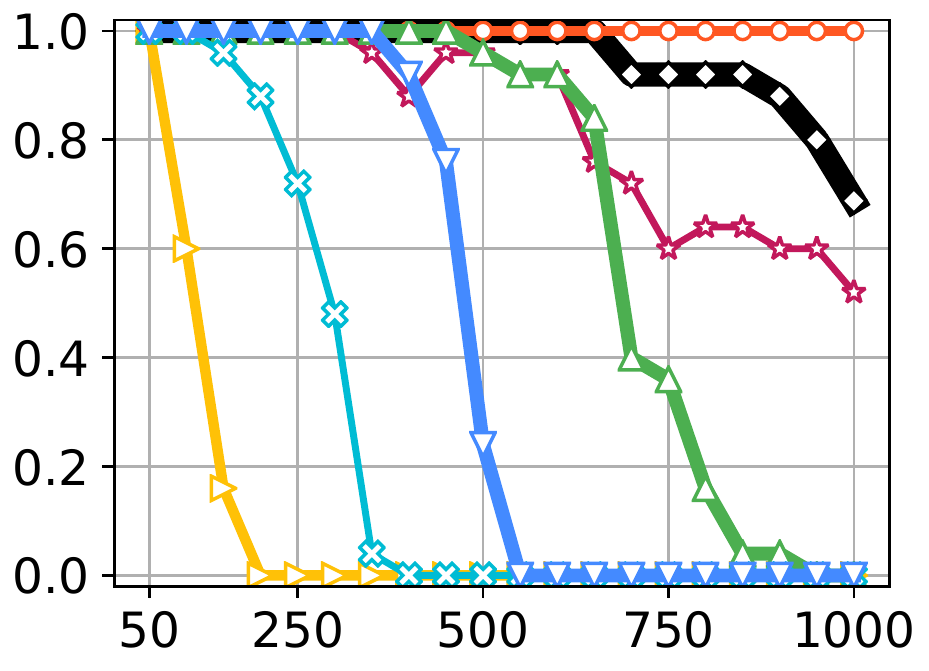}{251x180}{10,021} &
      \entry{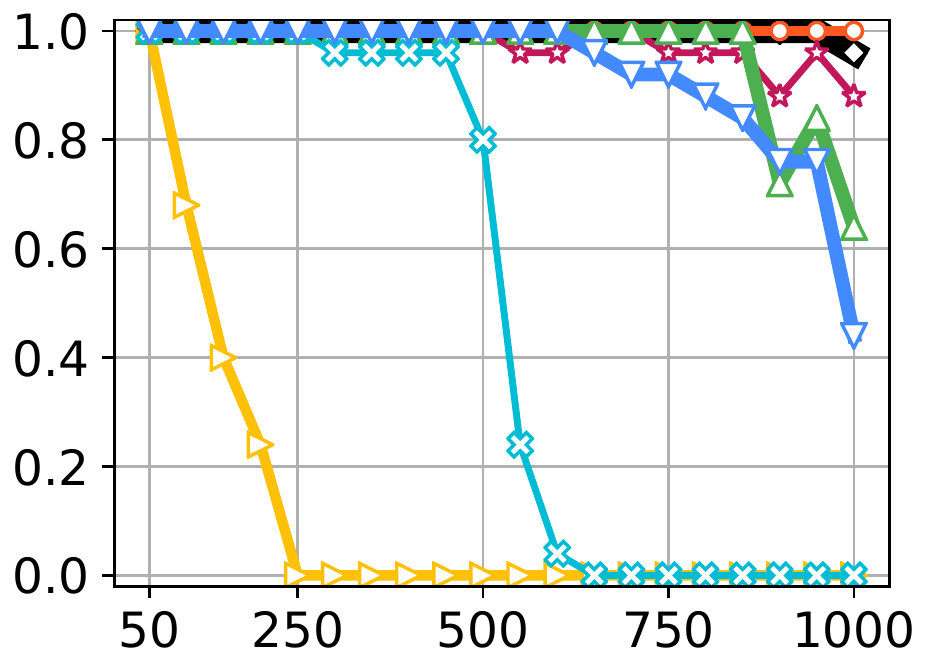}{530x481}{43,151} &
      \entry{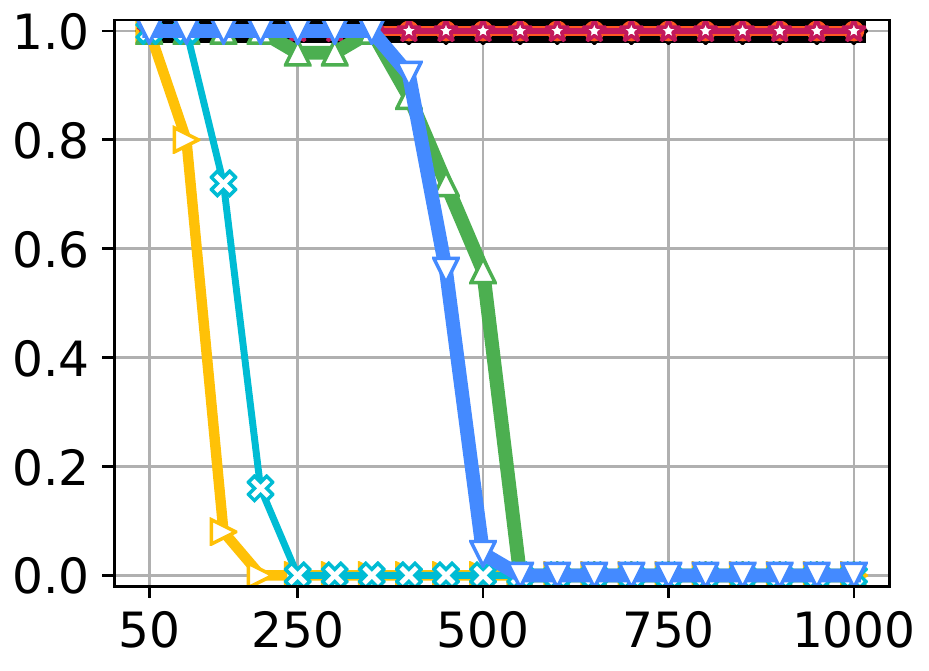}{1491x656}{96,603}
      \\
      \multicolumn{7}{c}{agents:~$|A|$}\smallskip\\
      \multicolumn{7}{c}{\includegraphics[width=0.65\linewidth]{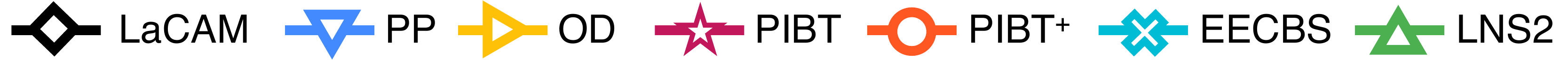}}
    \end{tabular}
    \caption{
      \textbf{Results of the MAPF benchmark.}
      The number of vertices for each grid is shown with parentheses.
      `cost' represents SOC divided by the total distance of start-goal pairs, $\sum_{i \in A}\dist{s_i}{g_i}$, where \dist{u}{v} is the shortest path length from $u \in V$ to $v \in V$ on the graph.
      This score works as the upper bound of sub-optimality, where the minimum is one.
      For `runtime' and `cost,' we show median scores of solved instances within each solver.
      We further show minimum and maximum scores using semi-transparent regions.
      The success rate of LaCAM was based on the number of successful trials over total trials.
    }
    \label{fig:result-main}
  \end{figure*}
}

\subsection{Experimental Setup}

\paragraph{Baselines}
We carefully selected the following six sub-optimal MAPF algorithms as baselines.
\begin{itemize}
\item \textbf{Prioritized Planning (PP)}~\cite{erdmann1987multiple,silver2005cooperative} as a basic approach for MAPF.
  PP used distance heuristics~\cite{van2005prioritized} for the planning order and \astar~\cite{hart1968formal} for single-agent pathfinding.
\item \textbf{\astar with operator decomposition (OD)}~\cite{standley2010finding} as an adaptation of the general search scheme to MAPF.
  We used a greedy search to obtain solutions as much as possible (i.e., neglecting g-value of \astar).
  The heuristic (i.e., h-value) was the sum of distance towards goals.
\item \textbf{PIBT}~\cite{okumura2022priority}, which repeats one-timestep planning to solve MAPF.
  We tested a vanilla PIBT because the LaCAM implementation used PIBT as a sub-procedure (\cref{sec:implementation}).
  To detect planning failure, we regarded that PIBT failed to solve an instance when it reached pre-defined sufficiently large timesteps.
\item \textbf{\pibtplus}~\cite{okumura2022priority} as a state-of-the-art scalable MAPF solver, which uses PIBT until a certain timestep.
  The rest of planning is performed by another MAPF algorithm.
  The complement phase used a rule-based solver, Push and Swap~\cite{luna2011push}.
\item \textbf{EECBS}~\cite{li2021eecbs} as a state-of-the-art search-based solver that bases on a celebrated MAPF algorithm, CBS~\cite{sharon2015conflict}.
  The sub-optimality was set to five to find solutions as much as possible.
  In \cref{sec:scalability}, it was set to the default value ($1.2$) of the authors' implementation due to the better performance.
\item \textbf{MAPF-LNS2 (LNS2)}~\cite{li2022mapf} as another excellent MAPF solver based on large neighborhood search.
\end{itemize}

It is worth mentioning that PP, PIBT$^{(+)}$, EECBS, and LNS2 are incomplete; they cannot detect unsolvable instances, unlike LaCAM.
For PIBT$^{(+)}$, EECBS, and LNS2, we used the implementations coded by their respective authors.%
\footnote{The codes are available on \url{https://github.com/{Kei18/pibt2, Jiaoyang-Li/EECBS, Jiaoyang-Li/MAPF-LNS2}}.}
For PP, we used an implementation included in~\cite{okumura2022priority}.
OD was own-coded in C++.

\paragraph{Evaluation Environment}
LaCAM was coded in C++, available in the online supplementary material.
The experiments were run on a desktop PC with Intel Core i9-7960X \SI{2.8}{\giga\hertz} CPU and \SI{64}{\giga\byte} RAM.
We performed a maximum of 32 different instances run in parallel using multi-threading.

\subsection{Small Complicated Instances}

First, we tested LaCAM with instances used in~\cite{luna2011push}, shown in \cref{table:small-complicated}.
The runtime limit was \SI{10}{\second}.
Since our LaCAM implementation used non-determinism (see \cref{sec:implementation}), it was run five times for the same instance while changing random seeds.

\cref{table:small-complicated} summarizes the results.
As reference records, we also show SOC-optimal solutions obtained by a vanilla \astar, ignoring the runtime limit.
Although most baseline methods failed several instances, LaCAM solved all the instances regardless of random seeds, within reasonable timeframes.
Regarding solution quality (i.e., SOC), LaCAM compromises the quality compared to PP, EECBS, and LNS2.
This is due to the nature of LaCAM, which progresses the search with a short planning horizon.

\subsection{MAPF Benchmark}
\label{sec:eval-mapf}

Next, we tested LaCAM using the MAPF benchmark~\cite{stern2019def}, which includes a set of four-connected grids and start–goal pairs for agents.
We selected twelve grids with different portfolios (e.g., size, sparseness, and complexity).
For each grid, 25 ``random scenarios'' were used while increasing the number of agents by $50$ up to the maximum.
Therefore, identical instances were tried for the solvers in all settings.
The runtime limit was set to \SI{30}{\second} following~\cite{stern2019def}.
LaCAM was run five times for each setting.
For reference, we report that \astar used in \cref{table:small-complicated} failed to solve an instance with ten agents in \emph{random-32-32-20}.

\Cref{fig:result-main} summarizes the results.
In most scenarios, LaCAM outperforms PP, OD, EECBS, and LNS2 in both success rate and runtime, while compromising the solution quality.
The runtime of LaCAM is comparable with PIBT$^{(+)}$, furthermore, LaCAM outperforms a vanilla PIBT in the success rate.
The most competitive results with LaCAM were scored by \pibtplus.
However, overall, the SOC scores of LaCAM are better than those of \pibtplus, especially in dense situations.
Furthermore, LaCAM solved challenging scenarios, such as \mapname{random-32-32-20} with 400 agents, where the baseline methods almost failed to solve.
In summary, LaCAM can solve many instances within short timeframes, with acceptable solution quality.
Meanwhile, we observed that LaCAM scored poor performance in several grids, such as \mapname{random-64-64-20} and \mapname{warehouse-20-40-10-2-1}.
We investigate this reason in the next.

\subsection{Adversarial Instance}
\label{sec:poor}

In the previous experiment, LaCAM quickly solved various scenarios but scored poor performance in several grids.
Specifically, LaCAM solved all instances of \mapname{warehouse-20-40-10-2-2} while it failed frequently in \mapname{warehouse-20-40-10-2-1}.
The two maps differ in the width of corridors: the former is two while the latter is one.
Therefore, we considered the existence of narrow corridors such that two agents cannot pass through could be a bottleneck of LaCAM.

For investigation, we prepared an adversarial instance for LaCAM (see \cref{table:dislike-example}), where two pairwise agents need to swap their locations in narrow corridors.
We counted the number of search iterations of the high-level search while changing the number of agents (two: only agents-$\{1, 2\}$ appear, four, and six).
LaCAM solved all instances, however, the search effort dramatically increases with more agents.
This is because, with more agents, the high-level search can contain a huge number of \emph{slightly different configurations}.
Those configurations differ only one or two agents differ in their locations and disturb the progression of the search.

From this observation, we consider the poor performance in several grids of \cref{fig:result-main} as follows.
The LaCAM implementation used a depth-first search style, therefore, once a configuration similar to \cref{table:dislike-example} appears during the search, resolving this configuration towards the goal configuration requires significant search effort.
Consequently, LaCAM reaches a timeout.
Overcoming this limitation is one promising future direction.
One resolution might be developing a better configuration generator other than PIBT.

{
  \setlength{\tabcolsep}{5pt}
  \begin{table}[tb!]
    \centering
    \begin{tabular}{crr}
      \toprule
      \multirow{4}{*}{
        \begin{minipage}{0.35\linewidth}
          \includegraphics[width=1\linewidth]{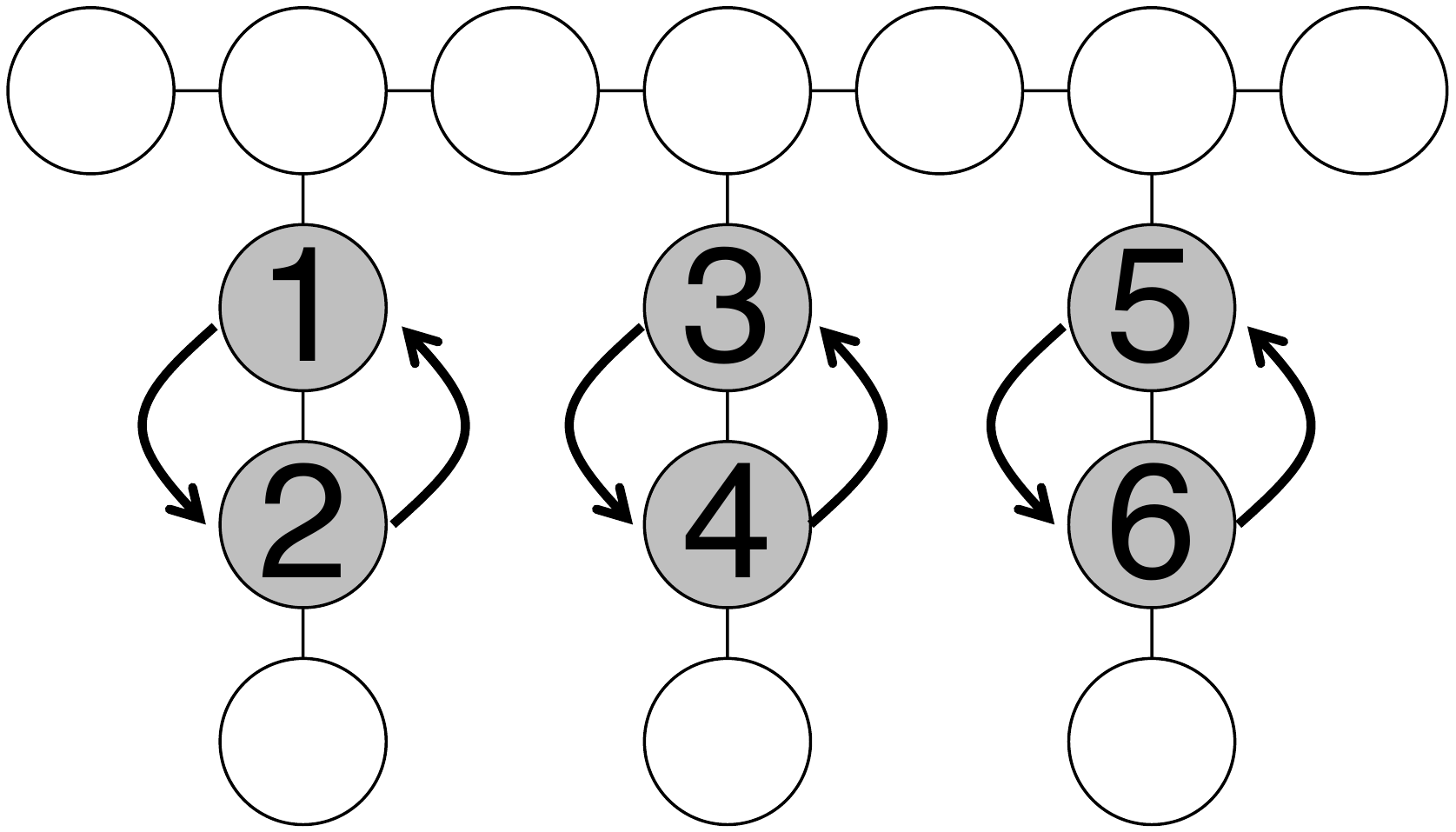}
        \end{minipage}
      }
      & $|A|$ & search iteration
      \\\cmidrule(lr){2-3}
      & 2 & 128
      \\
      & 4 & 23,907
      \\
      & 6 & 287,440
      \\\bottomrule
    \end{tabular}
    \caption{
      \textbf{LaCAM performance in an adversarial instance.}
    }
    \label{table:dislike-example}
  \end{table}
}

\subsection{Scalability Test}
\label{sec:scalability}
We evaluated the scalability of the number of agents, using instances with up to $10,000$ agents in \mapname{warehouse-20-40-10-2-2}.
The runtime limit was set to \SI{1000}{\second}.
OD was excluded since it run out of memory.
\Cref{fig:result-large} summarizes the result.
Only LaCAM solved all instances.
Furthermore, the runtime was in at most \SI{30}{\second} even with 10,000 agents, demonstrating the excellent scalability of LaCAM.%
\footnote{
We briefly report that LaCAM can be faster depending on computational environments.
As a pilot study, we tested the same setting with a single-thread run in a laptop with Intel Core i9 \SI{2.3}{\giga\hertz} CPU and \SI{16}{\giga\byte} RAM.
Even with 10,000 agents, LaCAM solved all instances at most in \SI{10}{\second} in the worst case.
}

{
  \setlength{\tabcolsep}{0pt}
  \begin{figure}[tb!]
    \centering
    \begin{tabular}{ccc}
      \begin{minipage}{0.32\linewidth}
        \includegraphics[width=1\linewidth]{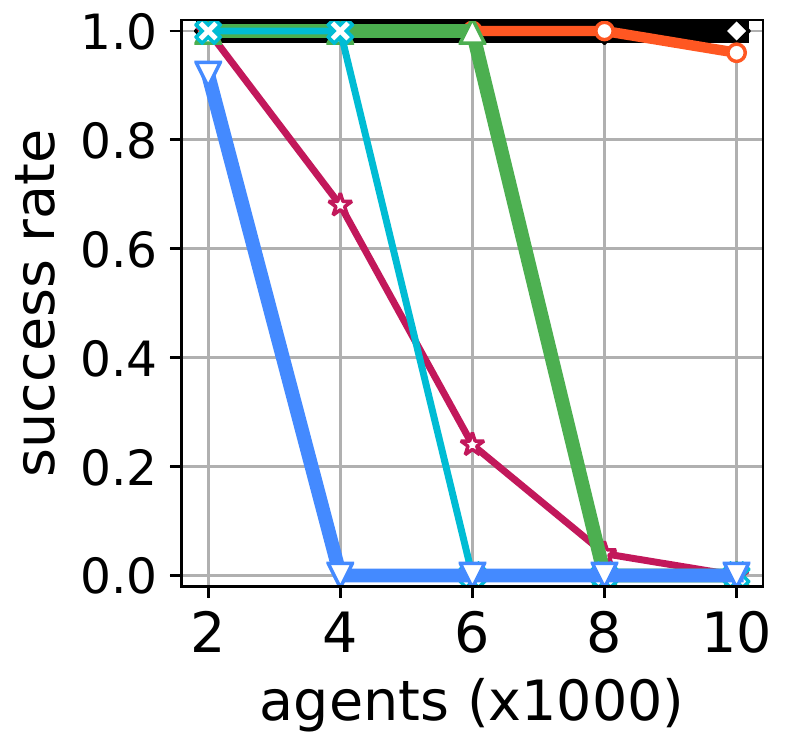}
      \end{minipage}
      &
      \begin{minipage}{0.32\linewidth}
        \includegraphics[width=1\linewidth]{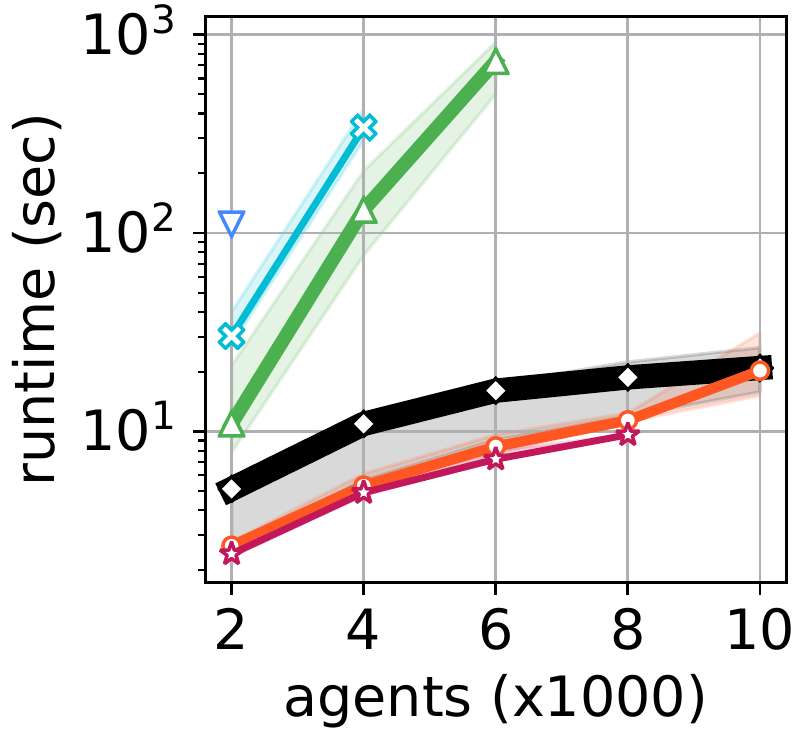}
      \end{minipage}
      &
      \begin{minipage}{0.32\linewidth}
        \includegraphics[width=1\linewidth]{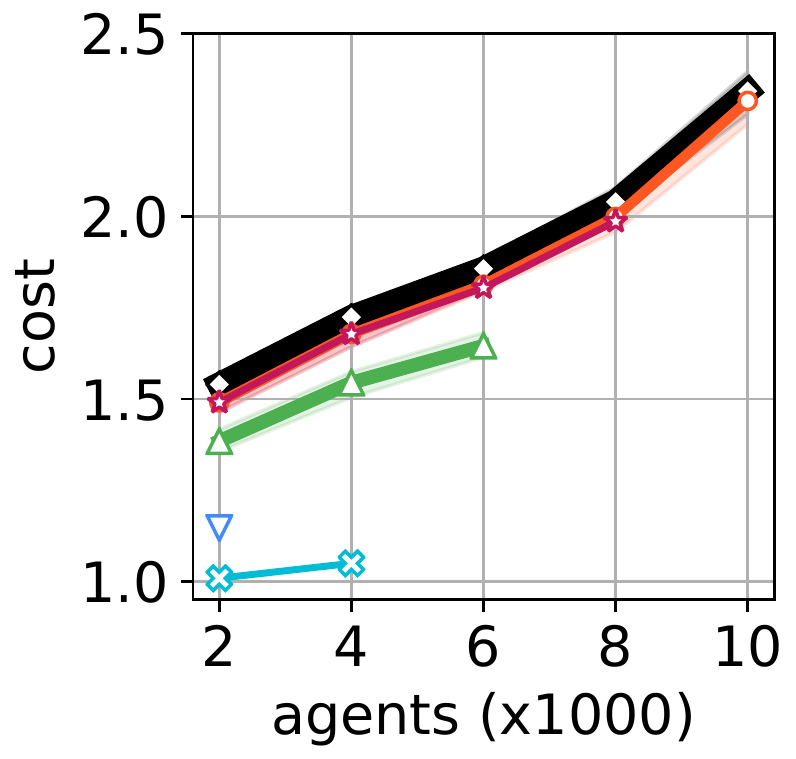}
      \end{minipage}
      \\
      \multicolumn{3}{c}{\includegraphics[width=0.9\linewidth]{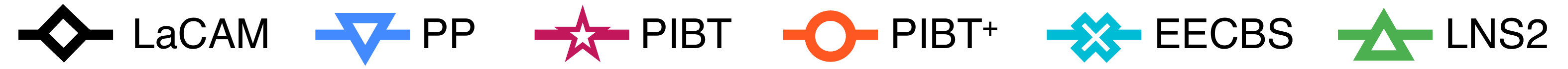}}
    \end{tabular}
    \caption{
      \textbf{Results with massive agents.}
      The used map was \mapname{warehouse-20-40-10-2-2}.
    }
    \label{fig:result-large}
  \end{figure}
}

{
  \setlength{\tabcolsep}{0pt}
  \begin{figure}[tb!]
    \centering
        \begin{tabular}{ccc}
      \begin{minipage}{0.32\linewidth}
        \includegraphics[width=1\linewidth]{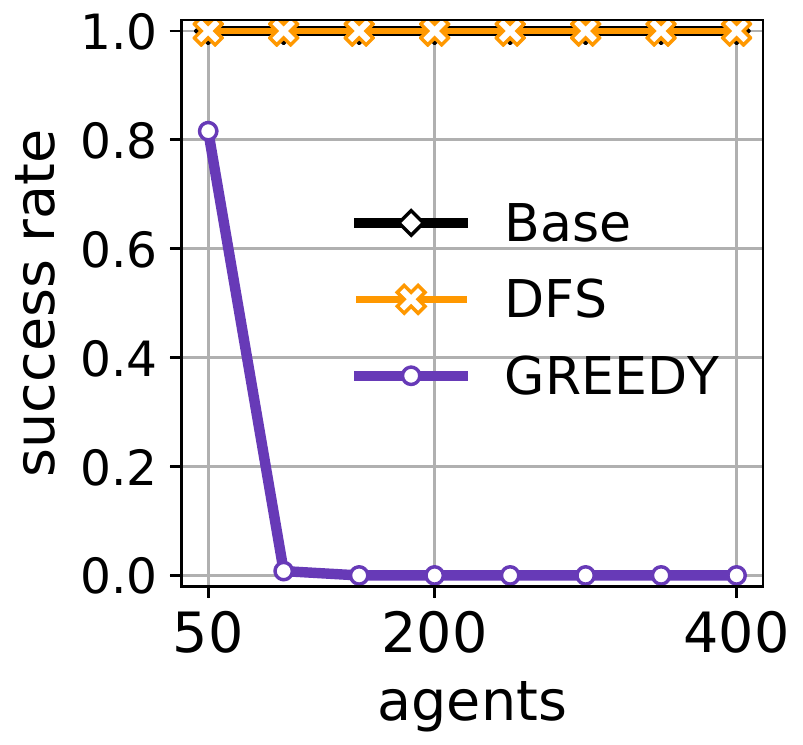}
      \end{minipage}
      &
      \begin{minipage}{0.34\linewidth}
        \includegraphics[width=1\linewidth]{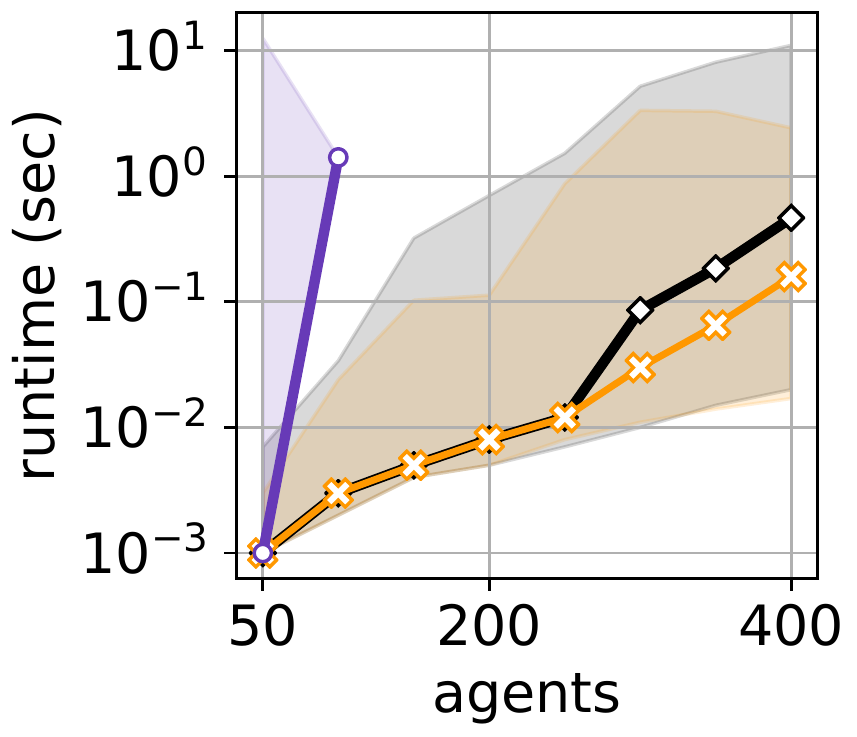}
      \end{minipage}
      &
      \begin{minipage}{0.32\linewidth}
        \includegraphics[width=1\linewidth]{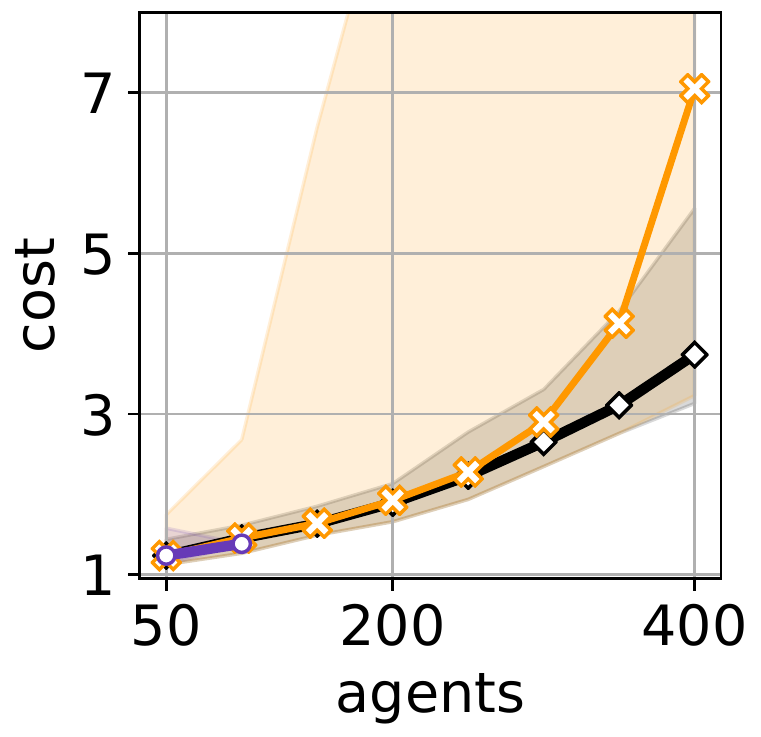}
      \end{minipage}
    \end{tabular}
    \caption{
      \textbf{Results with different LaCAM designs.}
      The used map was \mapname{random-32-32-20}.
    }
    \label{fig:result-design}
  \end{figure}
}

\subsection{Design Choice of LaCAM}
\label{sec:eval-design}

Finally, we investigated the design choice of LaCAM implementation.
Specifically, we assessed two variants:
\begin{itemize}
\item\textbf{DFS} does not use the reinsert operation at the high-level.
\item\textbf{GREEDY} uses another configuration generator instead of PIBT, such that agents greedily determines the location according to the priority order.
  This is equivalent to prioritized planning with a single-step planning horizon.
\end{itemize}

\Cref{fig:result-design} shows the result in \mapname{random-32-32-20}.
There are two observations:
1)~The reinsert operation improves the SOC metric.
2)~The choice of configuration generator significantly affects the search as seen in that GREEDY failed in most settings.
Consequently, the excellent performance of LaCAM in \cref{fig:result-main} relies on promising successor generation at the high-level, which is done by PIBT.

\section{Related Work}
\label{sec:related-work}

The multi-agent pathfinding (MAPF) problem~\cite{stern2019def} has been extensively studied since the 2010s.
LaCAM borrows several ideas from prior art.
We here review the relationship between LaCAM and existing MAPF algorithms.

In general, MAPF algorithms are categorized into four:
1)~\emph{Search-based} approaches search solutions in a coupled manner among agents~\cite{standley2010finding,sharon2015conflict}.
2)~\emph{Compiling-based} approaches reduce MAPF to well-known problems such as SAT~\cite{surynek2019unifying,lam2022branch}.
3)~\emph{Prioritized planning} approaches sequentially plan individual paths for agents in a decoupled manner~\cite{erdmann1987multiple,silver2005cooperative}.
4)~\emph{Rule-based} approaches make agents move step-by-step following ad-hoc rules~\cite{luna2011push}.
LaCAM is search-based, while the configuration generation is not limited to.
Indeed, PIBT~\cite{okumura2022priority} used in the experiments is categorized into prioritized planning and rule-based approaches.

Popular MAPF algorithms use a two-level search~\cite{sharon2013increasing,sharon2015conflict,surynek2019unifying,lam2022branch}.
At the high-level, these algorithms search \emph{negative} constraints that specify which agents cannot use where and when, while at the low-level, they perform single-agent pathfinding following constraints.
LaCAM also relies on a two-level search, however, constraints are addressed at the low-level.
Furthermore, our implementation uses \emph{positive} constraints that specify who to be where.
Positive constraints are not new in the literature;
we can see an example in CBS~\cite{li2019disjoint}.
Note that LaCAM with negative constraints is possible to implement, however, the search space for the low-level can be dramatically larger than using positive ones.

Apart from two-level search, \astar variants for MAPF have been also developed~\cite{standley2010finding,goldenberg2014enhanced,wagner2015subdimensional}.
In these studies, successors of a search node are generated without \emph{aggressive} coordination between agents beyond collision checking.
In contrast, LaCAM considers aggressive coordination when generating successors at the high-level, which is incorporated by the use of other MAPF algorithms as a configuration generator.

The manner of adding constraints is partially influenced by \astar with operator decomposition~\cite{standley2010finding}.
This algorithm creates successor nodes that correspond to one action of one agent, instead of actions for the entire agents.
LaCAM adds constraints in a similar scheme.

\section{Conclusion and Discussion}
\label{sec:conclusion}

This paper introduced LaCAM, a novel, complete, and quick search-based MAPF algorithm.
Our exhaustive experiments reveal that LaCAM can solve various instances in a very short time, even with complicated, dense, and challenging scenarios that other state-of-the-art sub-optimal solvers cannot handle.
Furthermore, LaCAM is scalable; even with 10,000 agents, it solved instances in tens of seconds.
In the following, we list discussions and future directions.

\paragraph{Why is LaCAM quick?}
In general, the average branching factor largely determines the search effort.
A vanilla \astar for MAPF generates $O(\Delta^{|A|})$ configurations from one search node, which is intractable especially when $|A|$ is large.
In contrast, each high-level node of LaCAM initially generates only one successor (i.e., configuration), and if necessary, gradually generates other successors in a lazy manner.
This scheme virtually suppresses the branching factor of LaCAM.
If the generated successor is promising (i.e., closer toward the goal configuration from the original), it can dramatically reduce the number of node generations.
Promising successors can be quickly obtained by techniques of recent MAPF studies such as PIBT.
This is a trick of quickness in LaCAM.
Although virtually reducing the branching factor in \astar for MAPF has been proposed~\cite{standley2010finding,goldenberg2014enhanced,wagner2015subdimensional}, these studies never achieve the dramatic reductions as LaCAM.

\paragraph{Optimal LaCAM}
This paper aimed at developing a quick MAPF solver;
we retain the discussion of optimality.
Since LaCAM is search-based, we consider that it is possible to develop optimal LaCAM.
Indeed, using a breadth-first search style instead of a depth-first style can solve makespan-optimal MAPF, where makespan is the maximum traveling time of agents.
We are also interested in developing an \emph{asymptotically optimal} version of LaCAM, that is, eventually converging to optima.
This is a common approach in motion planning algorithms~\cite{karaman2011sampling}.
Since solving MAPF optimally is NP-hard~\cite{yu2013structure}, such asymptotically-optimal approaches are practical for massive instances.
In the MAPF literature, such approaches have already appeared for other search-based algorithms~\cite{standley2011complete,cohen2018anytime}.

\paragraph{Improvements of Technical Components}
We investigated the design choices of LaCAM in \cref{sec:eval-design}, while other components should be further explored, such as how to add effective constraints in the low-level search and how to design effective configuration generators.
Those improvements will overcome the current limitation as seen in \cref{sec:poor}.

\section*{Acknowledgments}
I am grateful to Jean-Marc Alkazzi and Fran\c{c}ois Bonnet for their comments on the initial manuscript.
This work was partly supported by JSPS KAKENHI Grant Number 20J23011 and JST ACT-X Grant Number JPMJAX22A1.
I thank the support of the Yoshida Scholarship Foundation.

\bibliography{ref}

\end{document}